%% file: main.tex
\documentclass{article}

\usepackage{microtype}
\usepackage{graphicx}
\usepackage{subcaption}
\usepackage{booktabs} % for professional tables

\usepackage{xcolor}         % colors

\usepackage{hyperref}
\usepackage{url}            % simple URL typesetting

\usepackage[preprint]{icml2026}

\usepackage{amsmath,amsfonts}
\usepackage{amssymb}
\usepackage{mathtools}
\usepackage{amsthm}

\usepackage{bm}
\usepackage{bbm}
\usepackage{enumerate}
\usepackage{comment}
\usepackage{nicefrac}  
\usepackage{enumitem}

\usepackage[capitalize,noabbrev]{cleveref}

\theoremstyle{plain}
\newtheorem{theorem}{Theorem}[section]

\theoremstyle{definition}
\newtheorem{definition}[theorem]{Definition}

\theoremstyle{remark}

\renewcommand{\thefigure}{\thesection.\arabic{figure}}
\renewcommand{\thetable}{\thesection.\arabic{table}}

\makeatletter
\newcommand*\rel@kern[1]{\kern#1\dimexpr\macc@kerna}
\newcommand*\widebar[1]{%
  \begingroup
  \def\mathaccent##1##2{%
    \rel@kern{0.8}%
    \overline{\rel@kern{-0.8}\macc@nucleus\rel@kern{0.2}}%
    \rel@kern{-0.2}%
  }%
  \macc@depth\@ne
  \let\math@bgroup\@empty \let\math@egroup\macc@set@skewchar
  \mathsurround\z@ \frozen@everymath{\mathgroup\macc@group\relax}%
  \macc@set@skewchar\relax
  \let\mathaccentV\macc@nested@a
  \macc@nested@a\relax111{#1}%
  \endgroup
}
\makeatother

\newcommand{\Two}{I\hspace{-1.2pt}I}
\newcommand{\Three}{I\hspace{-1.2pt}I\hspace{-1.2pt}I}
\newcommand{\Four}{I\hspace{-1.2pt}V}

\allowdisplaybreaks[3]

\usepackage[textsize=tiny]{todonotes}

\icmltitlerunning{Pool-based Active Learning as Noisy Lossy Compression}

\begin{document}

\twocolumn[
  \icmltitle{Pool-based Active Learning as Noisy Lossy Compression: \\ Characterizing Label Complexity via Finite Blocklength Analysis}

  \icmlsetsymbol{equal}{*}

    \begin{icmlauthorlist}
        \icmlauthor{Kosuke Sugiyama}{waseda}
        \icmlauthor{Masato Uchida}{waseda}
    \end{icmlauthorlist}
    
    \icmlaffiliation{waseda}{Major in Computer Science and Communications Engineering, Waseda University, Tokyo, Japan}
    
    \icmlcorrespondingauthor{Kosuke Sugiyama}{kohsuke0322@asagi.waseda.jp}
    \icmlcorrespondingauthor{Masato Uchida}{m.uchida@waseda.jp}
    
    \icmlkeywords{information-theoretic learning theory, pool-based active learning, finite blocklength analysis, noisy lossy compression, label complexity, generalization error}
    
  \vskip 0.3in
]

\printAffiliationsAndNotice{}  % no special notice (required even if empty)

\begin{abstract}

This paper proposes an information-theoretic framework for analyzing the theoretical limits of pool-based active learning (AL), in which a subset of instances is selectively labeled.
The proposed framework reformulates pool-based AL as a noisy lossy compression problem by mapping pool observations to noisy symbol observations, data selection to compression, and learning to decoding.
This correspondence enables a unified information-theoretic analysis of data selection and learning in pool-based AL.
Applying finite blocklength analysis of noisy lossy compression, we derive information-theoretic lower bounds on label complexity and generalization error that serve as theoretical limits for a given learning algorithm under its associated optimal data selection strategy.
Specifically, our bounds include terms that reflect overfitting induced by the learning algorithm and the discrepancy between its inductive bias and the target task, and are closely related to established information-theoretic bounds and stability theory, which have not been previously applied to the analysis of pool-based AL.
These properties yield a new theoretical perspective on pool-based AL.

\end{abstract}

\section{Introduction}
\label{sec:intro}

Pool-based Active Learning (AL) selects instances from a finite unlabeled pool, queries their labels, and learns from the labeled data.
One objective of the pool-based AL theory is to characterize label complexity, defined as the minimum number of labels required for a target error.
Several frameworks analyze this based on the reduction of the hypothesis set through label observation, the complexity of the hypothesis set, and the properties of the data distribution; examples include the disagreement coefficient \cite{hanneke2007abound}, diameter-based AL \cite{tosh2017diameter}, and others \cite{kaariainen2006active,raginsky2011lower,zhang2014beyond,hanneke2015minimax,citovsky2021batch,gentile2022achieving,gentile2024fast,hanneke2025agnostic}.

Conversely, no established method evaluates theoretical limits for pool-based AL using properties of the learning algorithm, such as the degree of overfitting or inductive bias.
While such properties are central to passive learning frameworks like information-theoretic bounds (IT-bounds) \cite{xu2017information}, stability theory \cite{bousquet2002stability}, and PAC-Bayes \cite{mcallester1999pac}, which are essential for deep learning where analysis using the complexity of the hypothesis set fails \cite{zhang2017understanding,harutyunyan2021information,hardt2016train,dziugaite2017computing}, their extension to AL theory remains challenging.

In this paper, we focus on \cite{sugiyama2026finite} as a framework with the potential to resolve this issue.
This framework derives lower bounds on theoretical limits under an optimal training data sampling strategy, depending on properties of the learning algorithm such as the degree of overfitting and inductive bias.
By formulating the learning process as a lossy compression problem, this framework makes it possible to analyze training data construction and learning within the framework of finite blocklength analysis \cite{kostina2012fixed}. 
In this formulation, training data correspond to codewords, data sampling to encoding, and learning to decoding.
Within this framework, the derived lower bounds are closely related to quantities used to measure overfitting in IT-bounds and stability theory, unifying these perspectives by revealing how they emerge through a consideration of the sampling process.

However, although \cite{sugiyama2026finite} derives limits for optimal sampling, it cannot specifically analyze pool-based AL.
Since this framework assumes unconstrained sampling where instances are freely drawn from the input space, it cannot isolate specific sampling restrictions.
Consequently, the derived bounds become loose for pool-based AL, which is constrained to selection from a fixed pool, failing to capture its specific properties.
Thus, to focus on pool-based AL, we must extend the framework to incorporate this constraint.
Specifically, we need to model the data acquisition process not as arbitrary sampling, but as a composition of pool observation and data selection from the pool, while retaining the applicability of finite blocklength analysis.

To address this problem, we demonstrate that pool-based AL corresponds to noisy lossy compression \cite{kostina2016nonasymptotic}, and that the proposed extension is realized through its finite blocklength analysis.
Our first contribution lies in establishing this correspondence.
In noisy lossy compression, symbols are not observed directly; instead, a symbol passes through a {\it channel} to become a noisy symbol, which is subsequently encoded into a codeword.
This process can be viewed as decomposing the single encoding step of standard lossy compression into two distinct stages: the channel transmission and the encoding of the noisy symbol.
This decomposition aligns with the requirement to split the sampling strategy into pool observation and a data selection strategy.
Consequently, we map pool-based AL to noisy lossy compression as a process where a pool (noisy symbol) is obtained via observation (channel) according to the learning target distribution (source symbol), and training data (codeword) is acquired by selecting instances to label (encoding) from the pool.

Our second contribution is deriving lower bounds on label complexity and generalization error for pool-based AL with a given learning algorithm, leveraging the finite blocklength analysis of \cite{kostina2016nonasymptotic}.
Since their analysis fixes the channel to evaluate limits under optimal encoding, our application fixes the properties of pool (e.g., pool size) to derive limits achievable by optimal data selection strategy.
Our lower bounds reflect the properties of the pool and data selection, resulting in a representation specialized for pool-based AL.
To our knowledge, existing theories do not provide lower bounds on label complexity specific to pool-based AL with a finite pool; thus, establishing these limits via our lower bounds is a novelty of our work.
Furthermore, as a special case where the entire pool is labeled, we also cover i.i.d. sampling from the true distribution and derive the lower bound on its generalization error.

Our derived lower bound inherits advantages from \cite{sugiyama2026finite}, while remaining specialized for pool-based AL.
Unlike existing theories for pool-based AL that derive bounds based on the properties of the hypothesis set or distributional noise rates, our framework expresses these bounds via terms related to the degree of overfitting of the learning algorithm and its inductive bias mismatch.
This offers a novel perspective for describing the relationship between pool-based AL limits and properties of learning algorithms.
Moreover, the quantities in our bound relate to overfitting measures found in error upper bounds within IT-bounds and stability theory.
This extension thus bridges the theoretical gap between these theories and pool-based AL, demonstrating their potential applicability to AL theory.

\section{Preliminary}
\label{sec:preliminary}

\subsection{Machine Learning Process}
\label{subsec:machine_learning_process}

Let $\bm{X}$ and $Y$ be random variables representing instances and labels with domains $\mathcal{X}$ and $\mathcal{Y}$ and realizations $\bm{x}$ and $y$.
Let $P^*_{\bm{X}Y}(\bm{x},y)$ be the true distribution with support $\mathrm{supp}(P^{*}_{\bm{X}Y})$.
In the learning process, a sampling strategy obtain a training dataset $\{(\bm{x}_i,y_i)\}^n_{i=1}$ consisting of $n \in \mathbb{N}_+$ samples drawn from $P^*_{\bm{X}Y}(\bm{x},y)$. 
Subsequently, a learning algorithm produces a hypothesis $h:\mathcal{X} \rightarrow \mathcal{Y}$ using $\{(\bm{x}_i,y_i)\}^n_{i=1}$.
Let $H$ denote the random variable for the hypothesis and $\mathcal{H}$ be the hypothesis set. 
$h$ can also be denoted as $\widehat{P}^{h}_{Y|\bm{X}}$.
The learning objective is to obtain $h$ that minimizes a divergence $\mathcal{D}(P^*_{Y|\bm{X}} || \widehat{P}^{h}_{Y|\bm{X}})$ or an expected loss $\mathbb{E}_{P^*_{\bm{X}Y}}[l(h(\bm{X}),Y)]$, given a divergence $\mathcal{D}(\cdot || \cdot)$ or a loss function $l:\mathcal{Y}\times\mathcal{Y}\rightarrow \mathbb{R}_+$.

\subsection{Fixed-Length Lossy Compression}
\label{subsec:lossy_compression}

Fixed-length lossy compression \cite{berger1975rate} compresses a source symbol $A$, which follows a source distribution $P_A$, into one of $M \in \mathbb{N}_+$ codewords using an encoder $f$ and reconstructs it using a decoder $g$, allowing for reconstruction errors.
Here, the codelength, representing the degree of compression, is $\log M$ bits.
 Generally, block coding is adopted to compress a sequence of $k \in \mathbb{N}_+$ symbols $A^k=(A_1,\dots,A_k)$ collectively.
This framework analyzes the tradeoff between the rate $R= \log M /k$ and the distortion $\mathsf{d}(A^k;g(f(A^k)))$. 
While classical theory discusses the minimum rate in the limit of $k \rightarrow \infty$, finite blocklength analysis \cite{kostina2012fixed} analyzes the minimum rate for finite $k < \infty$.

\subsection{Learning Theory using Finite Blocklength Analysis of Lossy Compression}
\label{subsec:finite_blocklength_lc_ml}

\cite{sugiyama2026finite} maps the machine learning process to fixed-length lossy compression to analyze theoretical limits of generalization via the finite blocklength analysis of lossy compression \cite{kostina2012fixed}.

\textbf{Correspondence Between Machine Learning and Lossy Compression.}
Figure \ref{fig:correspondence_overview_extend} provides an overview.
They map codewords to the training dataset $\{(\bm{x}_i, y_i)\}^n_{i=1}$, while associating encoding (encoder) and decoding (decoder) with sampling (sampling strategy) and learning (learning algorithm), respectively.
To correspond to the codelength $\log M$ bits, assuming that a single sample $(\bm{x},y)$ is represented by $b \in \mathbb{N}_+$ bits, the total description length becomes $bn$ bits for the training dataset $\{(\bm{x}_i, y_i)\}^n_{i=1}$.
Next, to map the symbol to the distribution subject to sampling, an auxiliary variable $W \sim P_W$ with domain $\mathcal{W}$ and realization $w$, is introduced, and the input distribution $P^*_{\bm{X}}$ is decomposed as $P^*_{\bm{X}}(\bm{x})=\sum_{w \in\mathcal{W}}P^{*}_{\bm{X}|W}(\bm{x}|w)P_W(w)$.
As shown in Figure \ref{fig:correspondence_overview_extend}, this representation is always constructible via methods such as a lattice partition of $\mathrm{supp}(P^{*}_{\bm{X}})$, where $w$ indexes each region.
Introducing $W$ reformulates sampling from $P^*_{\bm{X}Y}$ as follows:
\begin{align}
    \underbracket[0.5pt]{w \sim P_W}_{\mathclap{\text{Determining } P^{*}_{\bm{X}Y|W=w}} } \rightarrow \underbracket[0.5pt]{\bm{x} \sim P^{*}_{\bm{X}|W=w} \rightarrow y \sim P^{*}_{Y|\bm{X}=\bm{x}}}_{\text{random sampling}}.
\label{eq:sampling_process_W}
\end{align}
This process entails two stages: determining the target $P^{*}_{\bm{X}Y|W}=P^{*}_{Y|\bm{X}}P^{*}_{\bm{X}|W}$ and sampling from it.
Generalizing the latter to arbitrary sampling, the symbol corresponds to $W$ (or $P^{*}_{\bm{X}Y|W}$), and the source distribution to $P_{W}$.
The distribution $P^{*}_{\bm{X}Y|W}$ is identified with $W$ and termed a \textit{sub-distribution} of $P^{*}_{\bm{X}Y}$.

Block coding corresponds to obtaining a total of $n$ training samples from $k \in \mathbb{N}_+$ sub-distributions $W^k \sim P_{W^k}=P_W\times\cdots\times P_W$ (i.e., $\{P^{*}_{\bm{X}Y|W=W_i}\}^k_{i=1}$) via an arbitrary allocation.
$k$ is termed the {\it sampling opportunities}, and the rate is defined as $R=b\frac{n}{k}$.
For a fixed $k$ and constant $b$, $R$ has a one-to-one correspondence with $n$, where $\frac{n}{k}$ represents the sampling efficiency (training data size per opportunity).
Distortion is mapped to the difference between $\widehat{P}^{*}_{Y|\bm{X}}$ and the hypothesis $P^{h}_{Y|\bm{X}}$ on $\mathrm{supp}(\{P^{*}_{\bm{X}|W=W_i}\}^k_{i=1})$.

\textbf{Notation.}
Let $T$ be the random variable for the training dataset with domain $\mathcal{T}$ and realization $t$ ($|t|=n$). 
Defining the sampling strategy as $\mathcal{S}:\mathcal{W}^k \rightarrow \mathcal{T}$ and the randomized learning algorithm as $\mathcal{A}:\mathcal{T} \rightarrow \mathcal{H}$, the machine learning process is described as $H=\mathcal{A}(\mathcal{S}(W^k))$.
The Markov kernels of $\mathcal{S}$ and $\mathcal{A}$ are denoted by $P^{\mathcal{S}}_{T|W^k}$ and $P^{\mathcal{A}}_{H|T}$, respectively.
The distortion measure can be defined as $\mathsf{d}(W;H) := \mathbb{E}_{P^*_{\bm{X}|W}}[\mathcal{D}(P^{*}_{Y|\bm{X}} || \hat{P}^{H}_{Y|\bm{X}})]$ for probabilistic models and $\mathsf{d}(W;H) := \mathbb{E}_{P^{*}_{Y|\bm{X}}P^*_{\bm{X}|W}}[l(Y,H(X))]$ for deterministic models.
For $k>1$, we define $\mathsf{d}(W^k;H) := \frac{1}{k} \sum^k_{i=1} \mathsf{d}(W_i;H)$. This quantity signifies the in-distribution generalization error, where the inaccessible $W$ is regarded as out-of-distribution.

\textbf{Overview of the Analysis.}
\cite{sugiyama2026finite} leverages the finite blocklength analysis \cite{kostina2012fixed} to determine the theoretical limits for a given $\mathcal{A}$.
Specifically, the study derives lower bounds on sample complexity and generalization error assuming an optimal sampling strategy $\mathcal{S}^*$ that minimizes the mutual information $\mathbb{I}(W;H)$ subject to the distortion constraint $\mathbb{E}[\mathsf{d}(W;\mathcal{A}(\mathcal{S}(W)))]\le d$.
Minimizing $\mathbb{I}(W;H)$ is equivalent to restricting information acquisition from $W$, thereby corresponding to the sampling strategy achieving the minimum training data size $n$.
The resulting bounds characterize the minimum $n$ to achieve a target $d$ and the minimum $d$ achievable with a fixed $n$.

\textbf{Problem.}
The previous lower bounds assume $\mathcal{S}^*$ can freely sample from $\mathrm{supp}(\{P^{*}_{\bm{X}Y|W=w_i}\}^k_{i=1})$ as there are no constraints other than $n$.
Therefore, they remain loose for pool-based AL and cannot capture the specific properties arising from finite pool constraints.
To address this issue, this paper focuses on noisy lossy compression \cite{kostina2016nonasymptotic} as a method to incorporate this constraint.

\subsection{Noisy Lossy Compression}
\label{subsec:finite_noisy_lossy_compression}

Noisy (fixed-length) lossy compression \cite{kostina2016nonasymptotic} considers a setting where a noisy symbol $\widebar{A} = c(A)$ is observed through a channel $c$, rather than the symbol $A$ itself.
Block coding assumes independent noise for each symbol, denoted as $\widebar{A}^k=(\widebar{A}_1, \dots, \widebar{A}_k)$ with $\widebar{A}_i = c(A_i)$.
The objective is for the encoder $f$ and decoder $g$ to minimize the distortion $\mathsf{d}(A^k;g(f(c(\widebar{A}^k))))$.

\subsection{Pool-based AL}
\label{subsec:pool-based_AL}

Pool-based AL \cite{settles2009active} assumes a pool of $m \in \mathbb{N}_+$ independent instances $\{\bm{x}'_j\}^m_{j=1}$ following the target distribution.
A training dataset is constructed by labeling $n \le m$ selected instances to derive a hypothesis.
The special case of $n=m$ is equivalent to i.i.d. sampling.
While selection and learning are practically iterative, we treat them as a sequential process without loss of generality.

\section{Proposed Theoretical Extension}
\label{sec:proposed_extension}

\begin{figure*}[ht]
    \centering
    \includegraphics[width=\linewidth]{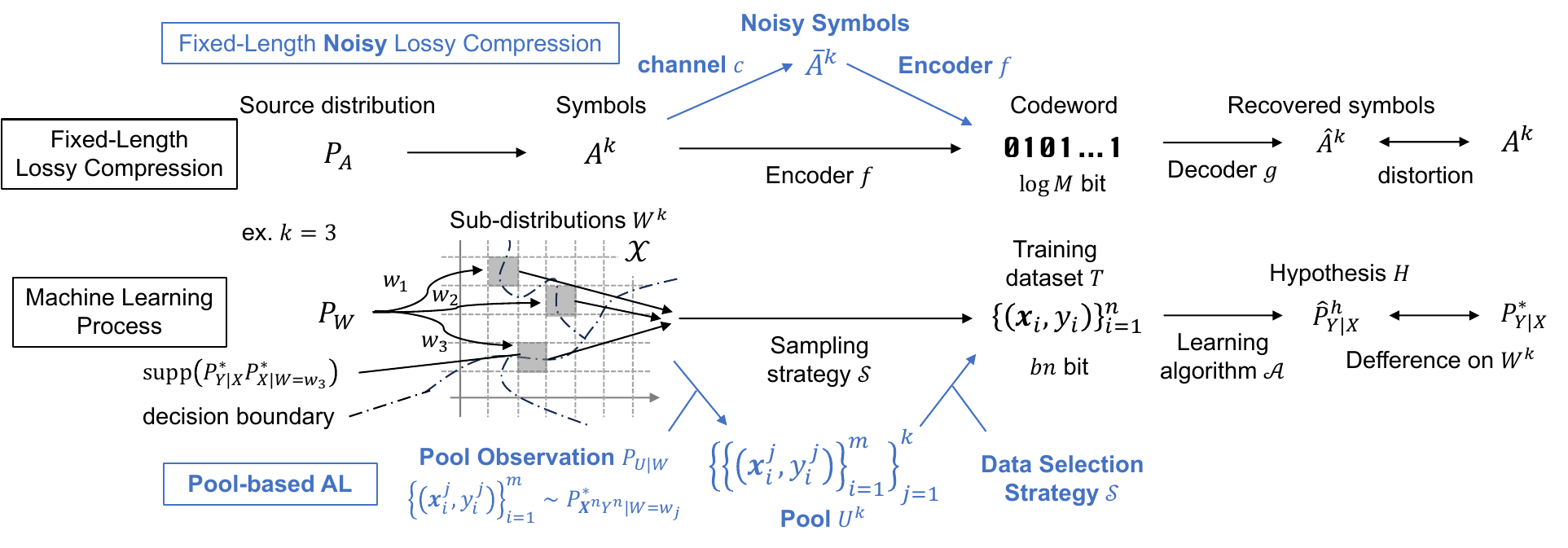}
    \caption{
             Overview of the correspondence in \cite{sugiyama2026finite} (black) and our extension (blue).
            }
    \label{fig:correspondence_overview_extend}
\end{figure*}

In this section, we establish a framework specialized for pool-based AL by extending \cite{sugiyama2026finite} via noisy lossy compression \cite{kostina2016nonasymptotic}.

\subsection{Idea of Extension}
\label{subsec:extension_idea}

We draw a parallel between the inability to directly observe $A^k$ in noisy lossy compression and the inability to freely sample from $\mathrm{supp}(\{P^{*}_{\bm{X}Y|W=w_i}\}^k_{i=1})$ in pool-based AL.
Standard (noiseless) lossy compression models the codeword generation from $A^k$ solely via the encoder $f$. 
Similarly, \cite{sugiyama2026finite} represented the acquisition of the training dataset $T$ from $W^k$ using the sampling strategy $\mathcal{S}$ alone.
In contrast, noisy lossy compression can be interpreted as decomposing this process into two stages: a fixed channel $c$ and the encoder $f$, which is the target of analysis.
Therefore, extending the framework of \cite{sugiyama2026finite} with noisy lossy compression allows us to incorporate the constraint by fixing the initial phase of the $W^k \to T$ process as the channel $c$.
Consequently, we can treat only the part corresponding to the encoder $f$ as the analysis target.

Focusing on training data construction in pool-based AL, this process is represented as two stages: observing a pool from the target distribution and subsequent data selection.
Since the former involves sampling according to the distribution, it is an operation that should remain fixed rather than being the subject of analysis; whereas the latter, data selection, is the main focus of pool-based AL analysis.
Therefore, by mapping pool observation to channel $c$ and data selection to encoder $f$, we formulate pool-based AL as noisy lossy compression, enabling a finite blocklength analysis centered on the selection process.

\subsection{Correspondence Between Pool-Based AL and Noisy Lossy Compression}
\label{subsec:pbal_nlc_correspondence}

Based on this idea, we specifically map pool-based AL to noisy lossy compression.
Figure \ref{fig:correspondence_overview_extend} illustrates the overview.
We adopt the mapping from \cite{sugiyama2026finite} excluding the encoding, and newly map the channel and encoding stages.
We first consider the case of $k=1$.

\textbf{Channel.}
The channel $c$ degrades the symbol $A$ into a noisy symbol $\widebar{A}$.
We map this channel to the observation of a pool $\{(\bm{x}'_j, y'_j)\}^m_{j=1}$ following the sub-distribution $W$ (i.e., $\{(\bm{x}'_j, y'_j)\}^m_{j=1} \sim P^{*}_{\bm{X}^mY^m|W} = P^{*}_{\bm{X}Y|W} \times \cdots \times P^{*}_{\bm{X}Y|W}$).
We then associate the noisy symbol with the pool $\{(\bm{x}'_j, y'_j)\}^m_{j=1}$.
Although the pool is originally unlabeled, we include labels $\{y'_j\}^m_{j=1}$ for convenience, assuming $y'_j$ is observed only if $\bm{x}'_j$ is selected as a training sample. 
This notational treatment entails no loss of generality.

\textbf{Encoding.}
Encoding constructs a codeword from noisy symbols $\widebar{A}^k$ to reconstruct the symbol $A$ with small distortion.
This corresponds to acquiring a training dataset from the pool to derive a low-error hypothesis for $P^{*}_{Y|\bm{X}}$ on $\mathrm{supp}(P^{*}_{\bm{X}|W})$.
Specifically, we map encoding to the selection of $n \le m$ samples from the pool of size $m$.
We refer to this function as the \textit{data selection strategy}.

\textbf{Block Coding.}
Next, we consider the case of $k > 1$.
Due to the independence of the channel, pool observation consists of observing $m$ samples from each sub-distribution $w_i$, resulting in a total of $km$ samples, from which the data selection strategy selects $n (\le km)$ samples.
When $m=1$, this results in an independent pool observation following $P^*_{\bm{X}Y}$, given by $\{(\bm{x}'_i, y'_i)\}^k_{i=1} \sim P^*_{\bm{X}^kY^k} = P^*_{\bm{X}Y} \times \dots \times P^*_{\bm{X}Y}$.
Furthermore, if we set $n=k$, the pool observation and data selection correspond to the construction of $T$ via i.i.d. sampling from $P^*_{\bm{X}Y}$.
Since the selection efficiency $\frac{n}{mk}$ is expressed as $R\frac{1}{bm}$ using the rate $R=b\frac{n}{k}$, we treat $R$ as synonymous with selection efficiency as well as with $n$.

\subsection{Additional Notation}
\label{subsec:nlc_additional_notation}

Notation excluding the sampling strategy follows Section \ref{subsec:finite_blocklength_lc_ml}.
We also assume the distortion measure $\mathsf{d}(\cdot;\cdot)$ is bounded (Assumption (I))
First, setting $k=1$, we define the following.
Let $U$ be the random variable for a pool of $m$ samples, $\mathcal{U}(W)$ be the family of $m$ samples that can be generated from $W (\equiv P^*_{\bm{X}Y|W})$, and $u \in \mathcal{U}(W)$ be the realization of $U$ from $W$.
We define the pool observation process as $P_{U|W} ~(= P^*_{X^m Y^m | W} = P^*_{XY|W} \times \cdots \times P^*_{XY|W})$, satisfying $\mathrm{supp}(P_{U|W}) \subseteq \mathcal{U}(W)$.
For $k>1$, we write the independent pool sequence as $U^k=(U_1,\dots, U_k) \sim P_{U^k|W^k} = \prod^k_{i=1}P_{U_i|W_i}$.
Let the data selection strategy be $\mathcal{S}:\prod^k_{i=1} \mathcal{U}(W_i) \rightarrow \mathcal{T}(U^k)$ with the Markov kernel $P^{\mathcal{S}}_{T|U^k}$ (where $\mathrm{supp}(P^{\mathcal{S}}_{T|U^k})=\mathcal{T}(U^k)$).
Here, the range $\mathcal{T}(U^k)$ is a family of subsets within the power set $\mathfrak{P}(U^k)$ of the $km$ samples in $U^k$; depending on whether $n$ is specified, it is $\mathcal{T}(U^k)=\mathfrak{P}_n(U^k)$ (family of sets of size $n$) or $\mathfrak{P}(U^k)$.
Thus, the overall pool-based AL process is described as $P^{\mathcal{A}}_{H|T} P^{\mathcal{S}}_{T|U^k}P_{U^k|W^k}$.

\section{Theoretical Analysis}
\label{sec:nlc_analysis}

Based on the correspondence in Section \ref{sec:proposed_extension}, we apply the finite blocklength analysis of noisy lossy compression \cite{kostina2016nonasymptotic} to pool-based AL.
Specifically, we derive lower bounds on label complexity and generalization error when using an optimal data selection strategy for a fixed randomized learning algorithm $\mathcal{A}$.

\subsection{Definition of Key Concepts}
\label{subsec:nlc_def_key_concepts}

First, we define the \textit{``good'' encoding (data selection strategy)} as our analysis target, along with the \textit{rate-distortion function} and \textit{tilted information} used for deriving lower bounds.
We adopt definitions from \cite{sugiyama2026finite}, while following \cite{kostina2016nonasymptotic} for differences arising from the extension.
Details are in Appendix \ref{apdx:sec:nlc_detail_interpret}.

\textbf{``Good'' Data Selection Strategy.}
For $k < \infty$, we define a class of ``good'' data selection strategy achieving small distortion with high probability via $\mathcal{A}$, along with the minimum training data size and minimum rate within this class
This class forms the subject of finite blocklength analysis.
\begin{definition}
  Given sampling opportunities $k$, a $(k, n, m, d, \epsilon, \mathcal{A})$ selection in $\{\mathcal{W}^k, \mathcal{H}, P_{W^k}, P_{U^k|W^k}, \mathsf{d}^k:\mathcal{W}^k\times\mathcal{H} \rightarrow [0,+\infty]\}$ is an $\mathcal{S}$ satisfying $m = |U_i|, \forall i \in \{1,\dots,k\}$ and $|\mathcal{S}(U^k)|=n$, subject to $\mathbb{P}[\mathsf{d}(W^k; \mathcal{A}(\mathcal{S}(U^k))) > d] \le \epsilon$ for distortion level $d$, excess distortion probability $\epsilon$, and algorithm $\mathcal{A}$.
  For $k=1$, it is denoted as an $(n, m, d, \epsilon, \mathcal{A})$ selection.
  The minimum training data size is defined as $n^*(k, m, d, \epsilon, \mathcal{A}) := \min\{ n: \exists (k, n, m, d, \epsilon, \mathcal{A}) \text{ selection} \}$, and the minimum rate as $R(k, m, d,\epsilon, \mathcal{A}) := \frac{b}{k}n^*(k,m, d,\epsilon, \mathcal{A})$.
  $n^*(k, m, d, \epsilon, \mathcal{A})$ is the minimum number of labels satisfying this condition, corresponding to label complexity in pool-based AL.
\label{def:nlc_k_n_m_d_e_A_selection}
\end{definition}

\textbf{Rate-Distortion Function $R(m,d,\mathcal{A})$.}
$R(m,d,\mathcal{A})$ denotes the minimum rate as $k \rightarrow \infty$, defined following \cite{kostina2012fixed}.
We first define the $\langle k, n, m, d, \mathcal{A} \rangle$ selection by substituting the condition in Definition \ref{def:nlc_k_n_m_d_e_A_selection} with $\mathbb{E}_{P_{U^k|W^k}P_{W^k}}[\mathsf{d}(W^k;\mathcal{A}(\mathcal{S}(U^k)))] \le d$, and let $R(k, m, d, \mathcal{A})$ denote its minimum rate.
$R(m,d,\mathcal{A})$ is then defined as:
\begin{align}
    \textstyle R(m,d, \mathcal{A}) := \limsup_{k\rightarrow\infty}R(k,m,d,\mathcal{A}).
\end{align}
Since this definition is based on \cite{kostina2012fixed}, it possesses the same properties as the original rate-distortion function.
Specifically, $R(m,d,\mathcal{A})$ is non-increasing and convex with respect to $d$ \cite{cover1999elements}.
From the i.i.d. property of $W$ and Assumption (I), for any $\epsilon \in (0,1)$, the equality $R(d,\mathcal{A}) = \limsup_{k\rightarrow \infty}R(k,m,d,\epsilon,\mathcal{A})$ holds, and $R(m,d,\mathcal{A})$ is equivalent to the following expression \cite{cover1999elements,kostina2016nonasymptotic}:
\begin{align}
\begin{split}
  &\textstyle \mathbb{R}_W(m,d,\mathcal{A}) \\
  &\textstyle:= \inf_{\mathcal{S} \in \mathfrak{S}: \mathbb{E}_{P_{U|W}P_W}[\bar{\mathsf{d}}(U;\mathcal{A}(\mathcal{S}(U)))] \le d}  \mathbb{I}(U;H) \\
  &\textstyle= \inf_{\mathcal{S}\in \mathfrak{S}: \mathbb{E}_{P^{\mathcal{A}}_{H|T}P^{\mathcal{S}}_{T|U}P_{U|W} P_W } [\bar{\mathsf{d}}(U;H)] \le d}  \mathbb{I}(U;H),
\end{split}
\label{eq:nlc_R_W}
\end{align}
Here, let $\mathfrak{S}:=\{\mathcal{S}:\mathrm{supp}(P^{\mathcal{S}}_{T|U}) \subseteq \mathcal{T}(U)\}$, and $\bar{\mathsf{d}}(u;h) = \mathbb{E}_{P_{W|U}}[\mathsf{d}(W;h) | U=u]$ denotes the expectation over the posterior $P_{W|U}$.
If regions on $\mathcal{X}$ partitioned by $W$ are disjoint and $w$ is uniquely determined by $u$, then $\bar{\mathsf{d}}(u;h) = \mathsf{d}(w;h)$.
Distinctions from \cite{sugiyama2026finite} include the dependence on $m$ and Eq. \eqref{eq:nlc_R_W}; whereas they define the corresponding quantity using $\mathbb{I}(W;H)$, we employ $\mathbb{I}(U;H)$ to explicitly reflect the properties of the pool.

To rule out exceptional $d$, we assume $d_{\mathrm{min}} := \inf\{d: \mathbb{R}_W(m,d,\mathcal{A}) < \infty \}$ exists and $d_{\mathrm{min}} < \infty$ (Assumption (\Two)).
To apply \cite{kostina2016nonasymptotic}, we further assume an optimal data selection strategy $\mathcal{S}^*$ exists that attains the infimum in Eq. \eqref{eq:nlc_R_W} exactly as $\mathbb{E}[\bar{\mathsf{d}}(U;\mathcal{A}(\mathcal{S^*}(U)))] = d$ (Assumption (\Three)).
This corresponds to selecting a minimal subset from the pool with average distortion $d$, plausible for sufficiently large $m$ and quantized $d$.

\textbf{Tilted Information.}
Tilted information is a quantity whose expectation equals the rate-distortion function.
In machine learning \cite{sugiyama2026finite}, it corresponds to the training data size ($\times b$) required to derive $h$ from $w$ via $\mathcal{A}$ and the optimal sampling strategy subject to average distortion $\le d$.
In information theory, it is interpreted as the minimum per-symbol codelength under the same constraint.
In this extension, following Eq. (9) of \cite{kostina2016nonasymptotic}, we define it such that its expectation equals $R(m,d,\mathcal{A})$:
\begin{definition}
    For any $\mathcal{A}$, $d > d_{\mathrm{min}}$, and $m$, we define the $\mathcal{A}$-specified $d$-tilted information at $(w, u, h)$ as follows:
    \begin{align}
    \begin{split}
        &\tilde{\jmath}_{W;U}(w,u,h,m,d,\mathcal{A})  \\
        &~~~~:= \iota_{U;H^{\mathcal{S}^*}_{\mathcal{A}}}(u;h) + \lambda^*_{\mathcal{A}}(d)(\mathsf{d}(w;h) - d),
    \end{split}
        \label{eq:ML_nlc_dfn_d-tilted}
    \end{align}
    where, $\iota_{U;H^{\mathcal{S}^*}_{\mathcal{A}}}(u;h):=\log \frac{\mathrm{d}P_{H^{\mathcal{S}^*}_{\mathcal{A}}U}}{\mathrm{d}P_{H^{\mathcal{S}^*}_{\mathcal{A}}} \mathrm{d}P_U}(u,h)$, $\lambda^{*}_{\mathcal{A}}(d) := -\frac{\mathrm{d} \mathbb{R}_W(m,d,\mathcal{A})}{\mathrm{d}d}> 0$, and $H^{\mathcal{S}^*}_{\mathcal{A}} = \mathcal{A}(\mathcal{S}^*(U^k))$.
\label{def:ML_nlc_dfn_d-tilted}
\end{definition}
The definition of $\iota_{U;H^{\mathcal{S}^*}_{\mathcal{A}}}(u;h)$ implies that $\mathbb{E}[\iota_{U;H^{\mathcal{S}^*}_{\mathcal{A}}}(U;H)]  = \mathbb{I}(U;H^{\mathcal{S}^*}_{\mathcal{A}}) = \mathbb{R}_W(m,d,\mathcal{A})$. 
Furthermore, since $\mathbb{E}_{P^{\mathcal{A}}_{H|T}P^{\mathcal{S}}_{T|U}P_{U|W} P_W }[\mathsf{d}(W;H)] = d$ is satisfied by the definition of $\mathcal{S}^{*}$, we obtain the relation $\mathbb{E}_{P^{\mathcal{A}}_{H|T}P^{\mathcal{S}}_{T|U}P_{U|W} P_W }[\tilde{\jmath}_{W;U}(W,U,H,m,d,\mathcal{A})] = \mathbb{R}_W(m,d,\mathcal{A})$.
To satisfy this requirement, while the tilted information in \cite{sugiyama2026finite} is defined using $\iota_{W;H}(w;h)$, this extension defines it using $\iota_{U;H^{\mathcal{S}^{*}}_{\mathcal{A}}}(u;h)$ instead. %due to the difference in the representation of $\mathbb{R}_W$.
Consequently, $\tilde{\jmath}_{W;U}$ depends on the pool $U$ and the data selection strategy $\mathcal{S}^{*}$, enabling the subsequent analysis using $\tilde{\jmath}_{W;U}$ to reflect concepts specific to pool-based AL.

\textbf{Interpretation of $\tilde{\jmath}_{W;U}$:} 
Since the definition of tilted information differs from that in \cite{sugiyama2026finite}, its interpretation changes accordingly.
In this extension, $\tilde{\jmath}_{W;U}(w,u,h,m,d,\mathcal{A})$ can be interpreted as the minimum number of selected samples ($\times b$ bits) required to obtain $h$ from the pool $u$ derived from $w$, using $\mathcal{S}^*$ and $\mathcal{A}$, while accounting for the constraint $\mathbb{E}[\mathsf{d}(W;\mathcal{A}(\mathcal{S}(U))] \le d$, for the following reason.
First, the $\iota_{U;H^{\mathcal{S}^*}_{\mathcal{A}}}(u,h)$, with expectation $\mathbb{I}(U;H^{\mathcal{S}^*}_{\mathcal{A}})=\mathbb{R}_W(m,d,\mathcal{A})$, corresponds to the per-$(u,h)$ minimum rate.
That is, it can be interpreted as the minimum training data size (number of selected samples) $\times b$ bits required to obtain $h$ from $u$ using $\mathcal{S}^*$ and $\mathcal{A}$.
On the other hand, $\lambda^*_{\mathcal{A}}(d)(\mathsf{d}(w;h) - d)$ represents a correction term according to the required distortion level $d$ \cite{sugiyama2026finite}.
This term is positive when $\mathsf{d}(w;h)>d$, increasing $\tilde{\jmath}_{W;U}$, and negative when $\mathsf{d}(w;h)<d$, decreasing it.
In other words, since additional data is required when $h$ fails to meet $d$, this term makes $\tilde{\jmath}_{W;U}$ larger than $\iota_{U;H^{\mathcal{S}^*}_{\mathcal{A}}}$ (and vice versa).
Thus, $\tilde{\jmath}_{W;U}$ represents the required minimum selected data size ($\times b$) ``considering $d$''.

\subsection{Lower Bound of Label Complexity}
\label{subsec:nlc_label_comp_bound}

In this section, we apply the finite blocklength analysis of \cite{kostina2016nonasymptotic} to analyze label complexity. 
Specifically, we show that given pool size $m$ and learning algorithm $\mathcal{A}$, the labeled data size cannot be further reduced to achieve a distortion $d$ for any data selection strategy.
As a preparation, we first present the following theorem corresponding to the essential Theorem 1 in \cite{kostina2016nonasymptotic}.
This theorem demonstrates a fundamental limit: for a required $d$ and fixed $m$ and $n$, the excess distortion probability $\epsilon$ cannot be improved upon, regardless of the data selection strategy.
See Appendix~\ref{apdx:subsec:proof_thm:ML_nlc_eps_conv_bound} for the proof.

\begin{theorem}
    The following holds for any $(n,m,d,\epsilon, \mathcal{A})$ selection:
    \begin{align}
        \epsilon \ge \sup_{\gamma \ge 0} \{ \mathbb{P}[\tilde{\jmath}_{W;U}(W,U,H,m,d,\mathcal{A}) \ge bn + \gamma] - e^{-\gamma}  \}.
        \label{eq:ML_nlc_eps_conv_bound}
    \end{align}
\label{thm:ML_nlc_eps_conv_bound}  
\end{theorem}
This theorem shows that the limit of $\epsilon$ in pool-based AL can be assessed by the probability that the labeled data size $n (\times b)$ drops below $\tilde{\jmath}_{W;U}$.
This result confirms the validity of interpreting $\tilde{\jmath}_{W;U}$ as the minimum selected data size.

As a final step, following \cite{kostina2012fixed}, we assume:
(\Four) $d \in (d_{\mathrm{min}}, d_{\mathrm{max}})$, where $d_{\mathrm{max}} := \sup \{d : \mathbb{R}_W(m,d,\mathcal{A})>0 \}$, and $\epsilon \in (0,1)$.
(V) $\tilde{\jmath}_{W;U}(w,u,h,m,d,\mathcal{A})$ has a finite absolute third moment.

Using Theorem \ref{thm:ML_nlc_eps_conv_bound}, we derive a lower bound for the minimum rate (or label complexity) $R(k,m,d,\epsilon,\mathcal{A})$ ($\propto n^*(k,m,d,\epsilon,\mathcal{A})$) as follows.
The following theorem indicates that for a specified pool size $m \times k$ and $\mathcal{A}$, the selected data size cannot be further decreased to satisfy $\mathbb{P}[\mathsf{d}(W^k; \mathcal{A}(\mathcal{S}(U^k))) > d] \le \epsilon$, irrespective of the data selection strategy.
See Appendix~\ref{apdx:subsec:proof_thm:ML_nlc_rate_conv_bound} for the proof.

\begin{theorem}
   Under Assumptions (I)--(V), for any $k$, $m$, $d \in (d_{\mathrm{min}},d_{\mathrm{max}})$, $\epsilon \in (0,1)$, $\mathcal{A}$, the following holds:
    \begin{align}
    \begin{split}
    &\textstyle R(k,m, d,\epsilon, \mathcal{A}) \\
    &\textstyle\ge R(m,d,\mathcal{A}) + \sqrt{\frac{V(m,d,\mathcal{A})}{k}}Q^{-1}(\epsilon) + O\big(\frac{\log k}{k}\big).
    \end{split}
    \label{eq:ML_nlc_rate_conv_bound}
  \end{align}
  Here, $Q(\epsilon)$ is the complementary cumulative distribution function of the standard normal distribution, and $Q^{-1}(\epsilon)$ is monotonically decreasing with respect to $\epsilon$.
  Furthermore, $V(m,d,\mathcal{A}) := \mathrm{Var}_{P^{\mathcal{A},\mathcal{S}^*}_{H|U} P_{U|W} P_W}(\tilde{\jmath}_{W;U}(W,U,H,m,d,\mathcal{A}))$, which is referred to as the rate-dispersion function.
\label{thm:ML_nlc_rate_conv_bound}
\end{theorem}

Theorem \ref{thm:ML_nlc_rate_conv_bound} enables a detailed analysis for pool-based AL, compared to Theorem 4.4 in \cite{sugiyama2026finite}, which presented limits among arbitrary sampling strategies.
First, since our theorem is an evaluation using the mean $R(m,d,\mathcal{A})$ and variance $V(m,d,\mathcal{A})$ of $\tilde{\jmath}_{W;U}$ dependent on the pool $U^k$ and the optimal data selection strategy $\mathcal{S}^*$, it can reflect the properties of pool-based AL such as $U^k$ and $\mathcal{S}^*$.
Setting $m=1$ yields $U^k=\{(\bm{x}'_j, y'_j)\}^k_{j=1} \sim P^{*}_{\bm{X}^k Y^k}=P^{*}_{\bm{X}Y}\times\cdots\times P^{*}_{\bm{X}Y}$; in this case, the lower bound corresponds to the case of generating a pool from $P^*_{\bm{X}Y}$.
To the best of our knowledge, no prior analysis of the lower bound of label complexity has focused on pool-based AL; our results establish this bound from a novel perspective.
In addition, while being specialized for pool-based AL, Eq. \eqref{eq:ML_nlc_rate_conv_bound} inherits two advantages similar to those of the lower bound in \cite{sugiyama2026finite}:
(i) $V(m,d,\mathcal{A})$ can be decomposed into two terms that separately evaluate the degree of overfitting of $\mathcal{A}$ and the degree of mismatch between the task and the inductive bias of $\mathcal{A}$.
(ii) Components of $R(m,d,\mathcal{A})$ and $V(m,d,\mathcal{A})$ correspond to overfitting measures in IT-bounds and stability theory.
We outline these below.
Detailed descriptions and proofs are provided in Appendix \ref{apdx:sec:detail_theorems} and Appendix \ref{apdx:subsec:derive_V_decomposition}.

First, advantage (i) is demonstrated by decomposing the variance of $V(m,d,\mathcal{A})$ with respect to $P_{W}$ as follows.
Here, $\tilde{\jmath}_{W;U}(w,u,h,m,d,\mathcal{A})$ is abbreviated as $\tilde{\jmath}_{W;U}$:
\begin{align}
    &V(m,d,\mathcal{A}) 
    \textstyle = V_{\mathrm{in}}(m,d,\mathcal{A}) + V_{\mathrm{bet}}(m,d,\mathcal{A}), \label{eq:ML_nlc_rate_dispersion_decomposition1} \\
    &\text{where~}V_{\mathrm{in}}(m,d,\mathcal{A}) :=\mathbb{E}_{P_W}\big[ \mathrm{Var}_{P^{\mathcal{A}}_{H|T}P^{\mathcal{S}^*}_{T|U}P_{U|W}} (\tilde{\jmath}_{W;U}) \big], \notag \\
    &~~~~~~~~~~~V_{\mathrm{bet}}(m,d,\mathcal{A}) :=\mathrm{Var}_{P_W}\big( \mathbb{E}_{P^{\mathcal{A}}_{H|T}P^{\mathcal{S}^*}_{T|U}P_{U|W}}[\tilde{\jmath}_{W;U}] \big). \notag
\end{align}

\textbf{Interpretation of $V_{\mathrm{in}}$ and $V_{\mathrm{bet}}$.}
Based on a perspective similar to \cite{sugiyama2026finite}, we can interpret $V_{\mathrm{in}}$ as the degree of overfitting of $\mathcal{A}$, and $V_{\mathrm{bet}}$ as the degree of inductive bias mismatch.
$V_{\mathrm{in}}$ is the expectation over $P_W$ of the variance of the required data size $\tilde{\jmath}_{W;U}$ caused by $P_{U|W}$, $P^{\mathcal{S}^*}_{T|U}$, and $P^{\mathcal{A}}_{H|T}$.
In the subsequent discussion, we assume that the randomness of the selection itself is minimal, since $\mathcal{S}^*$ selects optimal data.
Under this assumption, $V_{\mathrm{in}}$ can be interpreted as representing the instability of $\mathcal{A}$ arising from variations in $U$ and $T$ and the randomness inherent in $\mathcal{A}$.
As discussed later, $V_{\mathrm{in}}$ also relates to overfitting measures in IT-bounds and stability theory.
Thus, $V_{\mathrm{in}}$ serves as a measure of $\mathcal{A}$'s overfitting.
In contrast, $V_{\mathrm{bet}}$ is the variance over $P_W$ of $\tilde{\jmath}_{W;U}$ averaged over $P_{U|W}$, $P^{\mathcal{S}^*}_{T|U}$, and $P^{\mathcal{A}}_{H|T}$.
A large $V_{\mathrm{bet}}$ implies the coexistence of sub-distributions learnable with small data and those requiring large data.
Viewing the former as distributions where the inductive bias of $\mathcal{A}$ fits and the latter as distributions where it does not, a large $V_{\mathrm{bet}}$ can be interpreted as the inductive bias being mismatched in parts of $\mathrm{supp}(P^*_{\bm{X}Y})$.
As noted in \cite{sugiyama2026finite}, this interpretation aligns with the perspective that a well-suited inductive bias reduces the required data size \cite{canatar2021spectral,boopathy2023model,boopathy2024towards}, as well as with findings indicating that sample-wise learning difficulty varies depending on the inductive bias \cite{kwok2024dataset}.
Therefore, from this perspective, $V_{\mathrm{bet}}$ can be interpreted as measuring the degree of inductive bias mismatch.

Next, we outline advantage (ii) by decomposing $V_{\mathrm{in}}$, $V_{\mathrm{out}}$, and $R(m,d,\mathcal{A})$.

\textbf{Decomposition of $V_{\mathrm{in}}(m,d,\mathcal{A})$.}
While \cite{sugiyama2026finite} decomposes $V_{\mathrm{in}}$ into variances of sampling $P^{\mathcal{S}^*}_{T|W}$ and learning $P^{\mathcal{A}}_{H|T}$, our extension decomposes it into variances of pool $P_{U|W}$, data selection $P^{\mathcal{S}^*}_{T|U}$, and $P^{\mathcal{A}}_{H|T}$:
\begin{align}
    &V_{\mathrm{in}}(m,d,\mathcal{A}) 
    =\mathbb{E}_{P_{W}}[ 
    V_{\mathrm{in}, U}^{\iota} + V_{\mathrm{in}, \mathcal{S}}^{\iota} + V_{\mathrm{in}, \mathcal{A}}^{\iota} \notag \\
    &+(\lambda^*_{\mathcal{A}}(d))^2( V_{\mathrm{in}, U}^{\mathsf{d}} + V_{\mathrm{in}, \mathcal{S}}^{\mathsf{d}} + V_{\mathrm{in}, \mathcal{A}}^{\mathsf{d}} 
     )+ 2\lambda^{*}_{\mathcal{A}}(d) V_{\mathrm{in}}^{ \mathrm{cov}}], \notag \\
    &\text{where~}
     V_{\mathrm{in}, U}^{\iota} :=\mathrm{Var}_{P_{U|W}}( \mathbb{E}_{P^{\mathcal{A}}_{H|T} P^{\mathcal{S}^*}_{T|U}}[\iota_{U;H^{\mathcal{S}^*}_{\mathcal{A}}}(U;H)] ) , \notag\\
     &V_{\mathrm{in}, \mathcal{S}}^{\iota} :=\mathbb{E}_{P_{U|W}} [ \mathrm{Var}_{P^{\mathcal{S}^*}_{T|U}} ( \mathbb{E}_{P^{\mathcal{A}}_{H|T}} [\iota_{U;H^{\mathcal{S}^*}_{\mathcal{A}}}(U;H) ]) ] , \notag\\
     &V_{\mathrm{in}, \mathcal{A}}^{\iota} :=\mathbb{E}_{P^{\mathcal{S}^*}_{T|U}P_{U|W}}[ \mathrm{Var}_{P^{\mathcal{A}}_{H|T}}(\iota_{U;H^{\mathcal{S}^*}_{\mathcal{A}}}(U;H) )], \notag\\
     &V_{\mathrm{in}, U}^{\mathsf{d}} := \mathrm{Var}_{P_{U|W}}( \mathbb{E}_{P^{\mathcal{A}}_{H|T} P^{\mathcal{S}^*}_{T|U}} [\mathsf{d}(W;H)] ), \notag\\
     &V_{\mathrm{in}, \mathcal{S}}^{\mathsf{d}} := \mathbb{E}_{P_{U|W}}[ \mathrm{Var}_{P^{\mathcal{S}^*}_{T|U}} ( \mathbb{E}_{P^{\mathcal{A}}_{H|T}}[\mathsf{d}(W;H) ] ) ], \notag\\
     &V_{\mathrm{in}, \mathcal{A}}^{\mathsf{d}} := \mathbb{E}_{P^{\mathcal{S}^*}_{T|U}P_{U|W}}[ \mathrm{Var}_{P^{\mathcal{A}}_{H|T}}(\mathsf{d}(W;H) )], \notag\\
     &V_{\mathrm{in}}^{ \mathrm{cov}} :=  \mathrm{Cov}_{P^{\mathcal{A}}_{H|T}P^{\mathcal{S}^*}_{T|U}P_{U|W}} (\iota_{U;H^{\mathcal{S}^*}_{\mathcal{A}}}(U;H), \mathsf{d}(W;H) ). \notag
\end{align}

First, we explain $V_{\mathrm{in}, U}^{\iota}$ and $V_{\mathrm{in}, U}^{\mathsf{d}}$, newly introduced in this extension.
They represent the variances of the required data size $\iota_{U;H^{\mathcal{S}^*}_{\mathcal{A}}}(U;H)$ and distortion $\mathsf{d}(W;H)$ over $P_{U|W}$.
Since the empirical distribution of $U$ approaches $P^{*}_{\bm{X}Y|W}$ as pool size $m$ grows \cite{cover1999elements}, these terms should decrease.
On the other hand, if these terms remain large even for large $m$, considering that $\mathcal{S}^*$ always performs optimal data selection, it signifies that the small variation in $T=\mathcal{S}^*(U)$ accompanying $U$ significantly changes the data size $\iota_{U;H^{\mathcal{S}^*}_{\mathcal{A}}}(U;H)$ required by $\mathcal{A}$ and the distortion $\mathsf{d}(W;H)$ produced by $\mathcal{A}$.
Thus, these terms serve as measures of $\mathcal{A}$'s dependence on $T$.

$V_{\mathrm{in}, \mathcal{S}}^{\iota}$ and $V_{\mathrm{in}, \mathcal{A}}^{\iota}$ denote the variance of the required data size $\iota_{U;H^{\mathcal{S}^*}_{\mathcal{A}}}(U;H)$ caused by fluctuations of $T$ and by the intrinsic randomness of $\mathcal{A}$, respectively.
Since $\mathcal{S}^*$ is optimal, the variability of $T$ itself is expected to be small; hence, similarly to $V_{\mathrm{in}, U}^{\iota}$, $V_{\mathrm{in}, \mathcal{S}}^{\iota}$ can be interpreted as the sensitivity of $\mathcal{A}$ to small perturbations of $T$, and thus as another measure of its dependence on $T$.
In contrast, $V_{\mathrm{in}, \mathcal{A}}^{\iota}$ measures the instability of the required data size induced by the randomness of $\mathcal{A}$ itself.
Together with $V_{\mathrm{in}, U}^{\iota}$, these two terms are related to the overfitting measure of the IT-bound in \cite{hellstrom2020generalization}, denoted as $\mathrm{Var}(\iota_{T;H}(T;H))$. 
Specifically, while $\mathrm{Var}(\iota_{T;H}(T;H))$ assumes i.i.d. sampling, these three terms incorporate the corresponding quantity defined using $P_{U|W}$ and $\mathcal{S}^*$.
Details are provided in Appendix~\ref{apdx:subsubsec:V_in_iota_other_theories}.

$V_{\mathrm{in}, \mathcal{S}}^{\mathsf{d}}$ and $V_{\mathrm{in}, \mathcal{A}}^{\mathsf{d}}$ represent the variance of $\mathsf{d}(W;H)$ arising from fluctuations of $T$ and the inherent randomness of $\mathcal{A}$.
Although their interpretation parallels that of $V_{\mathrm{in}, \mathcal{S}}^{\iota}$ and $V_{\mathrm{in}, \mathcal{A}}^{\iota}$, they characterize $\mathcal{A}$ from a distinct perspective by measuring the variance of $\mathsf{d}(W;H)$ rather than $\iota_{U;H^{\mathcal{S}^*}_{\mathcal{A}}}(U;H)$.
These quantities are related to stability theory \cite{bousquet2002stability,elisseeff2005stability}, which analyzes generalization by bounding the variability of the error due to data perturbations or algorithmic randomness under i.i.d. sampling.
Accordingly, $V_{\mathrm{in}, \mathcal{S}}^{\mathsf{d}}$ and $V_{\mathrm{in}, \mathcal{A}}^{\mathsf{d}}$ can be seen as counterparts of stability-based variance measures defined using $P_{U|W}$ and $\mathcal{S}^*$.
Details are provided in Appendix~\ref{apdx:subsubsec:V_in_dist_other_theories}.

Finally, $V_{\mathrm{in}}^{\mathrm{cov}}$ captures the correlation between $\iota_{U;H^{\mathcal{S}^*}_{\mathcal{A}}}(U;H)$ and $\mathsf{d}(W;H)$.
Since hypotheses requiring more data are expected to achieve smaller $\mathsf{d}(W;H)$, this covariance term is generally negative.

\textbf{Decomposition of $V_{\mathrm{bet}}(m,d,\mathcal{A})$.}
From the definition of $\tilde{\jmath}_{W;U}$, $V_{\mathrm{bet}}(m,d,\mathcal{A})$ admits the following decomposition:
\begin{align}
    &V_{\mathrm{bet}}(m,d,\mathcal{A}) = V_{\mathrm{bet}}^{\iota} + (\lambda^{*}_{\mathcal{A}}(d))^2V_{\mathrm{bet}}^{\mathsf{d}} + 2\lambda^{*}_{\mathcal{A}}(d)V_{\mathrm{bet}}^{\mathrm{cov}}, \notag \\
    &\text{where~}
    V_{\mathrm{bet}}^{\iota}
     := \mathrm{Var}_{P_W}(  \mathbb{E}_{P^{\mathcal{A}}_{H|T}P^{\mathcal{S}^*}_{T|U}P_{U|W}}[\iota_{W;H^{\mathcal{S}^*}_{\mathcal{A}}}(W;H)  ] ), \notag \\
    &V_{\mathrm{bet}}^{\mathsf{d}} 
     := \mathrm{Var}_{P_W}(  \mathbb{E}_{P^{\mathcal{A}}_{H|T}P^{\mathcal{S}^*}_{T|U}P_{U|W}}[\mathsf{d}(W;H)] ), \notag \\
    &V_{\mathrm{bet}}^{\mathrm{cov}} 
     := \mathrm{Cov}_{P_W}( 
            \mathbb{E}_{P^{\mathcal{A}}_{H|T}P^{\mathcal{S}^*}_{T|U}P_{U|W}}[\iota_{W;H^{\mathcal{S}^*}_{\mathcal{A}}}(W;H)], \notag\\
            &~~~~~~~~~~~~~~~~~~~~~~~~~~~~~~~\mathbb{E}_{P^{\mathcal{A}}_{H|T}P^{\mathcal{S}^*}_{T|U}P_{U|W}}[\mathsf{d}(W;H)]). \notag
\end{align}

Each component captures a different aspect of the mismatch in the inductive bias of $\mathcal{A}$.
In particular, $V_{\mathrm{bet}}^{\iota}$ and $V_{\mathrm{bet}}^{\mathsf{d}}$ measure this mismatch through the variability of the required data size $\iota_{U;H^{\mathcal{S}^*}_{\mathcal{A}}}(U;H)$ and the distortion $\mathsf{d}(W;H)$ over $\mathrm{supp}(P^{*}_{\bm{X}Y})$, respectively.
The remaining $V_{\mathrm{bet}}^{\mathrm{cov}}$, similarly to $V_{\mathrm{in}}^{\mathrm{cov}}$, is expected to be negative.

\textbf{Regarding $R(m,d,\mathcal{A})$.}
The term $R(m,d,\mathcal{A})$ in \eqref{eq:ML_nlc_rate_conv_bound} decomposes as $R(m,d,\mathcal{A})=\mathbb{I}(U;H)=\mathbb{I}(T;H) - \mathbb{I}(T;H|U)$.
Here, $\mathbb{I}(T;H)$ corresponds to the overfitting measure in IT-bounds \cite{xu2017information}, linking $R(m,d,\mathcal{A})$ directly to IT-bounds.
While $\mathbb{I}(T;H)$ of \cite{xu2017information} assumes i.i.d. sampling, $R(m,d,\mathcal{A})$ accounts for data selection, serving as its pool-based AL counterpart.
Details are provided in Appendices~\ref{apdx:subsec:nlc_detail_rate_distortion_function} and \ref{apdx:subsec_bridge_rate_dist_others_nlc}.

\subsection{Lower Bound of Generalization Error}
\label{subsec:nlc_error_bound}

Next, we derive a lower bound on the generalization error realizable by $\mathcal{A}$ and $\mathcal{S}^*$ at a given rate $R$ (or $n$).
For finite $k$, let $D(k,m,R,\epsilon,\mathcal{A})$ and $D(k,m,R,\mathcal{A})$ denote the minimum distortion achievable by $(k,m,n,d,\epsilon,\mathcal{A})$ selection and $\langle k,m,n,d,\mathcal{A} \rangle$ selection, respectively.
We define the {\it distortion-rate function} $D(R,m,\mathcal{A}):= \limsup_{k \rightarrow \infty} D(k,m,R,\mathcal{A})$ to represent the asymptotic minimum distortion as $k \rightarrow \infty$.
For fixed $k$, $m$, and $\epsilon$, the functions $D(k,m,\cdot,\epsilon,\mathcal{A})$ and $R(k,m,\cdot,\epsilon,\mathcal{A})$ are inverse to each other \cite{kostina2016nonasymptotic}.
Similarly, $D(m,\cdot,\mathcal{A})$ is also the inverse of $R(m,\cdot,\mathcal{A})$ \cite{cover1999elements}.

Based on Theorem~\ref{thm:ML_nlc_rate_conv_bound}, we establish the following lower bound on $D(k,m,R,\epsilon,\mathcal{A})$.
This bound provides a fundamental limit: for fixed $m$ and $\mathcal{A}$, no data selection strategy can reduce the generalization error beyond this bound using only $n$ label queries.
The proof is given in Appendix~\ref{apdx:subsec:proof_thm:ML_nlc_distortion_conv_bound}.

\begin{theorem}%[Lower bound of $D(k,R,\epsilon, \mathcal{A})$]
In addition to Assumptions (I)--(V), suppose that $R(m,d,\mathcal{A})$ is twice differentiable with $\frac{\mathrm{d}R(m,d,\mathcal{A})}{\mathrm{d}d} \neq 0$, and that $V(m,d,\mathcal{A})$ is differentiable on an interval $(\underline{d}, \bar{d}] \subseteq (d_{\mathrm{min}}, d_{\mathrm{max}}]$.
Then, for any rate $R$ such that $R = R(m,d,\mathcal{A})$ for some $d \in (\underline{d}, \bar{d})$, the following holds:
\begin{align}
\begin{split}
    &\textstyle D(k,m,R,\epsilon, \mathcal{A}) 
    \textstyle \ge D(m,R,\mathcal{A}) + \\
    &\textstyle \sqrt{\frac{\mathcal{V}(m,R,\mathcal{A})}{k}}Q^{-1}(\epsilon) +D'(m,R,\mathcal{A}) +  O \big( \frac{\log k}{k} \big),
\end{split}
  \label{eq:ML_nlc_distortion_conv_bound} \\
&\mathcal{V}(m,R,\mathcal{A})) := (D'(m,R,\mathcal{A}))^2  V(m,D(m,R,\mathcal{A}), \mathcal{A}). \notag
\end{align}
\label{thm:ML_nlc_distortion_conv_bound}
\end{theorem}

This theorem allows us to derive lower bounds on distortion specialized to pool-based AL or to i.i.d. sampling, depending on the choices of $m$ and $n$.
When $m=1$, we have $U^k=\{(\bm{x}'_j, y'_j)\}^k_{j=1} \sim P^{*}_{\bm{X}^k Y^k}=P^{*}_{\bm{X}Y}\times\cdots\times P^{*}_{\bm{X}Y}$, and the bound corresponds to the minimum distortion achievable by pool-based AL that selects samples from a pool $U^k$ drawn according to $P^{*}_{\bm{X}Y}$.
Furthermore, setting $m=1$ and $n = k$ (i.e., $R = \frac{bn}{n} = b$) yields the bound for constructing $T$ via i.i.d. sampling from $P^{*}_{\bm{X}Y}$.

Since $\mathcal{V}(m,R,\mathcal{A})$ in \eqref{eq:ML_nlc_distortion_conv_bound} is defined using $V(m,d,\mathcal{A})$ defined in Theorem~\ref{thm:ML_nlc_rate_conv_bound}, the interpretation of $\mathcal{V}(m,R,\mathcal{A})$ directly follows from that of $V(m,d,\mathcal{A})$.
In other words, the lower bound on the generalization error can be discussed in terms of the degree of overfitting and the mismatch of inductive bias of $\mathcal{A}$, thereby revealing connections to IT-bounds and stability theory.

\section{Related Work}
\label{sec:related_work}

Existing theoretical analyses of pool-based AL derive theoretical bounds based on the properties of the hypothesis set.
The main paradigm investigates the behavior of the hypothesis set as it shrinks with sequential label observations.
Examples include diameter-based AL \cite{tosh2017diameter}, which focuses on the diameter of the surviving hypothesis set, and the disagreement coefficient \cite{hanneke2007abound,puchkin2021exponential}, which quantifies the region of disagreement among surviving hypotheses.
These approaches derive label complexity and generalization error bounds that depend on quantities such as the complexity of the hypothesis set (e.g., VC dimension) and the noise conditions of the data distribution.
This characteristic is also common to other theoretical frameworks for pool-based AL \cite{zhang2014beyond,gentile2022achieving,gentile2024fast} as well as those analyzing the theoretical limits of arbitrary AL \cite{kaariainen2006active,hanneke2007teaching,castro2008minimax,raginsky2011lower,hanneke2014theory,hanneke2015minimax,citovsky2021batch,yuan2024regimes,hanneke2025agnostic}.

In contrast, our framework establishes theoretical limits for optimal data selection from a new perspective, capturing the degree of overfitting and inductive bias mismatch for a given algorithm.
Our lower bound exhibits a direct relationship with IT-bounds \cite{xu2017information,hellstrom2020generalization} and stability theory \cite{bousquet2002stability,elisseeff2005stability}.
This connection demonstrates the potential applicability of these theories to the theoretical analysis of pool-based AL.
Furthermore, the aforementioned studies do not provide theoretical limits specific to pool-based AL with a finite pool. Establishing these limits is a key novelty of this paper.

\section{Conclusion and Future work}
\label{sec:conclusion}

We established a framework to evaluate theoretical limits focused on pool-based AL by mapping it to noisy lossy compression and leveraging finite blocklength analysis.
This framework derives lower bounds on label complexity and generalization error for a randomized learning algorithm $\mathcal{A}$, assuming an optimal data selection strategy $\mathcal{S}^*$ from a pool $U$.
In particular, the lower bound on generalization error can also be derived for the special case where the training dataset is constructed via i.i.d. sampling.
These bounds quantify the degree of overfitting of $\mathcal{A}$ and the mismatch between its inductive bias and the task as distinct metrics, while also revealing connections to IT-bounds and stability theory.
These results present a novel analytical direction, distinct from conventional pool-based AL theories that focus on the properties of the hypothesis set.

\textbf{Future Work:}
Being an extension of \cite{sugiyama2026finite}, our framework inherits its limitations: (i) the computational challenges of the lower bounds, and (ii) the difficulty in validating the inductive bias mismatch captured by $V_{\mathrm{bet}}$.
Addressing these issues and broadening the application scope remain crucial future directions.

\section*{Impact Statement}

This paper aims to advance learning theory by applying the finite-length analysis of noisy lossy compression to the theory of pool-based AL.
We anticipate that further developments of this work will contribute to both the development and deployment of pool-based AL methods.
Regarding potential negative societal impacts, such as ethical considerations, we do not foresee any specific concerns at this stage.

\section*{Acknowledgment}
This work was supported in part by the Japan Society for the Promotion of Science through Grants-in-Aid for Scientific Research (C) (23K11111).

\bibliography{reference}
\bibliographystyle{icml2026}

\newpage
\appendix
\onecolumn

\input{appendix}

\end{document}

%% file: appendix.tex
\section{Detailed Interpretation of Concepts}
\label{apdx:sec:nlc_detail_interpret}

In this section, we provide an interpretation of the rate-distortion function, along with supplementary explanations regarding tilted information, both of which were introduced in Section \ref{subsec:nlc_def_key_concepts}.

\subsection{Interpretation of the Rate-Distortion Function}
\label{apdx:subsec:nlc_detail_rate_distortion_function}

While the specific interpretation of the rate-distortion function $R(m,d,\mathcal{A})$ is fundamentally consistent with that in \cite{sugiyama2026finite}, the definition of $\mathbb{R}_{W}(m, d,\mathcal{A})$ differs.
Consequently, we provide a renewed discussion here to account for this distinction.
We base our explanation on $\mathbb{R}_{W}(m, d,\mathcal{A})$, which is equivalent to $R(m,d,\mathcal{A})$ under the assumption of the i.i.d. property of $W$ and Assumption (I):
\begin{align}
\begin{split}
  &\mathbb{R}_W(d, \mathcal{A}) 
  := \inf_{\mathcal{S}: \mathbb{E}[\mathsf{d}(W;\mathcal{A}(\mathcal{S}(W)))] \le d}  \mathbb{I}(W;H) \\
  &~~~~~= \inf_{P^{\mathcal{S}}_{T|W}: \mathbb{E}_{P^{\mathcal{A}}_{H|T}P^{\mathcal{S}}_{T|W} P_W } [\mathsf{d}(W;H)] \le d}  \mathbb{I}(W;H).
\end{split}
\tag{Eq.~\eqref{eq:nlc_R_W}}
\end{align}

Following \cite{sugiyama2026finite}, we examine the mutual information $\mathbb{I}(U;H)$ used in the preceding equation.
It can be shown that this term is rewritten as follows, given the Markov chain $U \rightarrow T \rightarrow H$:
\begin{theorem}
  For any distributions $P_W$, $P_{U|W}$, $P^{\mathcal{S}}_{T|W}$, and $P^{\mathcal{A}}_{H|T}$, the following holds:
  \begin{align}
    & \mathbb{I}(U;H) = \mathbb{I}(U;T) - \mathbb{I}(U;T|H),  \label{eq:I-UT_UTH} \\
    & \mathbb{I}(U;H) = \mathbb{I}(T;H) - \mathbb{I}(T;H|U).  \label{eq:I-TH_THU}
  \end{align}
\label{thm:nlc_rate_dist_MI_transform}
\end{theorem}

This theorem is proved by replacing $W$ with $U$ in the proof of Theorem A.3 in \cite{sugiyama2026finite} as follows.

\begin{proof}[Proof of Theorem \ref{thm:nlc_rate_dist_MI_transform}]
By the chain rule, the following holds for $\mathbb{I}(U;T,H)$:
\begin{align}
    \mathbb{I}(U;T,H) = \mathbb{I}(U;H) + \mathbb{I}(U;T|H) = \mathbb{I}(U;T) + \mathbb{I}(U;H|T). \notag
\end{align}
Here, since $\mathbb{I}(U;H|T)=0$ due to the Markov property $U \rightarrow T \rightarrow H$, the above equation yields the following, which establishes Eq. \eqref{eq:I-UT_UTH}:
\begin{align}
    \mathbb{I}(U;H) = \mathbb{I}(U;T) - \mathbb{I}(U;T|H). \notag
\end{align}
Similarly, the chain rule for $\mathbb{I}(T;U,H)$ implies:
\begin{align}
    \mathbb{I}(T;U,H) = \mathbb{I}(T;U) + \mathbb{I}(T;H|U) = \mathbb{I}(T;H) + \mathbb{I}(T;U|H). \notag
\end{align}
Rearranging the terms yields:
\begin{align}
    \mathbb{I}(U;T|H) = \mathbb{I}(U;T) + \mathbb{I}(T;H|U) - \mathbb{I}(T;H). \notag
\end{align}
Substituting this into Eq. \eqref{eq:I-UT_UTH} proves Eq. \eqref{eq:I-TH_THU}:
\begin{align}
    \mathbb{I}(U;H) = \mathbb{I}(T;H) - \mathbb{I}(T;H|U).
\end{align}    
\end{proof}

The expression $\mathbb{I}(U;H) = \mathbb{I}(U;T) - \mathbb{I}(U;T|H)$ given by Eq. \eqref{eq:I-UT_UTH} allows for an intuitive interpretation of the meaning of $\mathbb{I}(U;H)$ as follows.
Specifically, $\mathbb{I}(U;H)$ is decomposed into two terms: $\mathbb{I}(U;T)$, which represents the amount of information extracted by the data selection strategy $\mathcal{S}$, and $\mathbb{I}(U;T|H)$, which represents the amount of information missed by the learning algorithm $\mathcal{A}$.
The first term $\mathbb{I}(U;T)$ denotes the mutual information between the pool $U$ and the training dataset $T$ selected by $\mathcal{S}$, implying how much information about $U$ is contained in $T$.
In other words, $\mathbb{I}(U;T)$ indicates how much information $\mathcal{S}$ has obtained from the pool on average.

The second term $\mathbb{I}(U;T|H)$ represents the mutual information between $U$ and $T$ conditioned on $H$.
This signifies the amount of information about $U$ that the hypothesis $H$, obtained by inputting $T$ into $\mathcal{A}$, failed to capture from $T$.
If $\mathcal{A}$ is a simple learning algorithm such as a linear regression model, it can capture linear input-output relationships on $T$ but cannot capture nonlinear ones.
In this case, $\mathbb{I}(U;T|H)$ is expected to be large.
On the other hand, if $\mathcal{A}$ is a deep learning model, it has the potential to capture such nonlinear relationships; thus, $\mathbb{I}(U;T|H)$ is expected to be small.
Therefore, $\mathbb{I}(U;T|H)$ can be interpreted as reflecting the complexity of the learning algorithm $\mathcal{A}$.

In the framework of \cite{sugiyama2026finite}, $\mathbb{R}_{W}$ is defined using $\mathbb{I}(W;H)$ and is expressed as $\mathbb{I}(W;H)=\mathbb{I}(W;T) - \mathbb{I}(W;T|H)$ by a similar decomposition.
Regarding the first term, \cite{sugiyama2026finite} considers the mutual information between $W$ and $T$, measuring the extent of information extraction by the sampling strategy $\mathcal{S}:W \rightarrow T$.
In contrast, the proposed framework considers the mutual information between $U$ and $T$, measuring the extent of information extraction focusing specifically on data selection from a pool.
This effectively captures the shift in analytical focus from arbitrary sampling strategies to data selection strategies from a pool.
As for the second term, \cite{sugiyama2026finite} measures the degree of information about $W$ missed by $T$, whereas the proposed framework measures the degree of information about $U$ missed by $T$.
This reflects the nature of pool-based AL, where the pool is already observed and its observation cannot be controlled; thus, only the data selection strategy can be optimized.

Based on the interpretation of $\mathbb{I}(U;H)$ provided by Eq. \eqref{eq:I-UT_UTH}, $\mathbb{R}_W$ defined in Eq. \eqref{eq:nlc_R_W} can be interpreted as follows.
$\mathbb{R}_W$ is the value obtained by minimizing $\mathbb{I}(U;H)$ with respect to $\mathcal{S}$ subject to the constraint that the average distortion is at most $d$.
Assuming the learning algorithm is sufficiently complex, the second term $\mathbb{I}(U;T|H)$ in $\mathbb{I}(U;H)$ is expected to be small, and its variation due to $\mathcal{S}$ is considered insignificant.
Under this simplifying assumption, $\mathbb{R}_W$ primarily corresponds to the minimization of the first term $\mathbb{I}(U;T)$ with respect to $\mathcal{S}$.
It is considered that $\mathbb{I}(U;T)$ increases as the number of labeled samples increases.
Therefore, $\mathbb{R}_W$ can be interpreted as representing the minimum amount of information about $U$ that $T$ must contain to achieve an average distortion of $d$ or less.
Furthermore, the optimal data selection strategy $\mathcal{S}^*$ that achieves the infimum in Eq. \eqref{eq:nlc_R_W} can be regarded as the data selection strategy that requires the minimum number of labeled samples while satisfying the average distortion constraint.

\subsection{Supplementary Explanation of Tilted Information}
\label{apdx:subsec:nlc_detail_tilted_information}

In this section, we provide a supplementary explanation of the $\mathcal{A}$-specified $d$-tilted information defined in Section \ref{subsec:nlc_def_key_concepts} as follows:
\begin{align}
    \tilde{\jmath}_{W;U}(w,u,h,m,d,\mathcal{A})  := \iota_{U;H^{\mathcal{S}^*}_{\mathcal{A}}}(u;h) + \lambda^*_{\mathcal{A}}(d)(\mathsf{d}(w;h) - d) \tag{Eq.~\eqref{eq:ML_nlc_dfn_d-tilted}}
\end{align}

The focus of this section is on how $\tilde{\jmath}_{W;U}$ varies depending on $\mathcal{A}$.
While the fundamental explanation remains consistent with the discussion of tilted information in the Appendix of \cite{sugiyama2026finite}, our framework introduces a new dependence on $u$. 
Accordingly, we provide a renewed explanation to account for this change.
Specifically, following the explanation in \cite{sugiyama2026finite}, we describe the behavior of the components of $\tilde{\jmath}_{W;U}$ influenced by $\mathcal{A}$: $\iota_{U;H^{\mathcal{S}^*}_{\mathcal{A}}}(u;h)$ and $\lambda^*_{\mathcal{A}}(d)$.

\textbf{Relationship between $\iota_{U;H^{\mathcal{S}^*}_{\mathcal{A}}}(u;h)$ and $\mathcal{A}$.}
First, by definition, $\iota_{U;H^{\mathcal{S}^*}_{\mathcal{A}}}(u;h)$ can be expressed as $\iota_{U;H^{\mathcal{S}^*}_{\mathcal{A}}}(u;h)=\log\frac{\mathrm{d} P^{\mathcal{A},\mathcal{S}^*}_{H|U=u}}{\mathrm{d} P^{\mathcal{A},\mathcal{S}^*}_{H}}(h)$, where $P^{\mathcal{A},\mathcal{S}^*}_{H|U=u}(h|u) = \sum_{t \in \mathcal{T}(u)} P^{\mathcal{A}}_{H|T}(h|t) P^{\mathcal{S}^*}_{T|U}(t|u)$ and $P^{\mathcal{A},\mathcal{S}^*}_{H}(h) = \sum_{w \in \mathcal{W}} \sum_{u \in \mathcal{U}(w)} P^{\mathcal{A},\mathcal{S}^*}_{H|U=u}(h|u) P_{U|W}(u|w)P_W(w)$.
From this expression, $\iota_{U;H^{\mathcal{S}^*}_{\mathcal{A}}}(u;h)$ measures the extent to which the hypothesis $h$, generated from the pool $u$ via $\mathcal{S}^*$ and $\mathcal{A}$, is specialized to $u$.
Since obtaining a hypothesis more specialized to a certain pool $u$ requires selecting more training samples from $u$, the interpretation that $\iota_{U;H^{\mathcal{S}^*}_{\mathcal{A}}}(u;h)$ represents the required number of training data (i.e., the number of selected data) remains consistent in this expression.

Based on this expression of $\iota_{U;H^{\mathcal{S}^*}_{\mathcal{A}}}(u;h)$, we can intuitively interpret the relationship between $\mathcal{A}$ and $\iota_{U;H^{\mathcal{S}^*}_{\mathcal{A}}}(u;h)$.
The value of $\iota_{U;H^{\mathcal{S}^*}_{\mathcal{A}}}(u;h)$ becomes large when $h$ is output with high probability by $\mathcal{A}$ and is specialized to $u$.
When the relationship between $\bm{X}$ and $Y$ in $u$ is complex, such a hypothesis $h$ can be produced by an algorithm $\mathcal{A}$ employing deep neural networks capable of generating complex hypotheses, whereas it is relatively less likely to be produced by an $\mathcal{A}$ restricted to simple models such as linear regression.
Therefore, $\iota_{U;H^{\mathcal{S}^*}_{\mathcal{A}}}(u;h)$ can be interpreted as a measure of the complexity of the hypotheses that $\mathcal{A}$ is capable of generating.

\textbf{Relationship between $\lambda^*_{\mathcal{A}}(d)$ and $\mathcal{A}$.}
The term $\lambda^*_{\mathcal{A}}(d)$ is defined as $\lambda^{*}_{\mathcal{A}}(d)=-\frac{\mathrm{d} \mathbb{R}_W(d,\mathcal{A})}{\mathrm{d}d}$.
$\mathbb{R}_W(d,\mathcal{A})$ corresponds to the minimum number of labeled data size required to achieve an average distortion of at most $d$ using $\mathcal{S}^*$ and $\mathcal{A}$.
Since $\mathbb{R}_W(d,\mathcal{A})$ is convex and monotonically non-increasing with respect to $d$, reducing the required distortion level $d$ results in an increase in $\mathbb{R}_W(d,\mathcal{A})$.
Consequently, $\lambda^*_{\mathcal{A}}(d)$, defined as the negative gradient of $\mathbb{R}_W(d,\mathcal{A})$, represents the sensitivity of the required number of labeled data size to small variations in the specified $d$.
In particular, as the specified $d$ becomes smaller, the number of labeled data size required to achieve an even smaller distortion level is considered to increase; thus, $\lambda^*_{\mathcal{A}}(d)$ is considered to increase as $d$ decreases.
This implies that the smaller the specified $d$, the greater the influence of the distortion penalty (or reward) in $\tilde{\jmath}_{W;U}$.

Furthermore, the increase in the required number of labeled data size when reducing the target distortion level $d$ varies across different algorithms $\mathcal{A}$.
For example, if the inductive bias of $\mathcal{A}$ is well-suited to the task and generalizes efficiently, the increase in training data required to achieve a smaller $d$ is considered to be suppressed.
Conversely, if $\mathcal{A}$ employs a neural network architecture that fails to capture the task structure, a substantial amount of additional labeled data may be required to achieve a smaller $d$.
Therefore, the value of $\lambda^*_{\mathcal{A}}(d)$ is considered likely to be relatively small when $\mathcal{A}$ is resistant to overfitting or when its inductive bias matches the task.

\section{Relationship between Derived Lower Bounds and Other Theoretical Frameworks}
\label{apdx:sec:detail_theorems}

The lower bound of the label complexity derived in Section \ref{subsec:nlc_label_comp_bound} was expressed in terms of the rate-distortion function $R(m,d,\mathcal{A})$ and the rate-dispersion function $V(m,d,\mathcal{A})$.
In this section, we discuss the relationship between these quantities and existing frameworks, which was not fully covered in Section \ref{subsec:nlc_label_comp_bound}.
The extension of the framework of \cite{sugiyama2026finite} presented in this paper inherits the advantage of their original framework regarding its connection to existing theories.
Therefore, while the explanation of the fundamental relationships below follows that of \cite{sugiyama2026finite}, we restate it here to explicitly account for the dependence on the pool $U$ and the data selection strategy $\mathcal{S}$, which is unique to our extension.

As a preliminary step, following \cite{sugiyama2026finite}, we define the i.i.d. sampling $\mathrm{S}^{\mathrm{iid}}$ as follows:
\begin{definition}[Definition B.1. in \cite{sugiyama2026finite}]
    The i.i.d. sampling $\mathrm{S}^{\mathrm{iid}}$ is a sampling strategy that satisfies the following conditions:
    (i) For a number of sampling opportunities $k$, the size $n$ of the generated training dataset $T$ is equal to $k$.
    (ii) When the number of sampling opportunities is $k>1$, the dataset $T=\{(\bm{X}_i,Y_i)\}^k_{i=1}$ is sampled as follows:
    \begin{align}
        (\bm{X}_i,Y_i) \sim P^*_{\bm{X}Y|W=w_i}, ~~~ \forall i \in \{1,\dots,k\}.
    \end{align}
    We assume that the generation of $T$ by $\mathrm{S}^{\mathrm{iid}}$ can be represented as a Markov kernel $P^{\mathrm{S}^{\mathrm{iid}}}_{T|W}$.
    When we explicitly indicate that $H$ is generated from $P^{\mathcal{A}}_{H|T}P^{\mathrm{S}^{\mathrm{iid}}}_{T|W}$, we denote it as $H^{\mathrm{S}^{\mathrm{iid}}}_{\mathcal{A}}$.
\label{def:iid_sampling}
\end{definition}
Since $W^k$ follows $P_W$ independently, in the training dataset $T=\{(\bm{X}_i,Y_i)\}^n_{i=1}$ obtained by $\mathrm{S}^{\mathrm{iid}}$, each sample follows $P^*_{\bm{X}Y}$ independently.
Therefore, we denote $T \sim P^{\mathrm{S}^{\mathrm{iid}}}_{T|W}$ as $T \sim P^*_{\bm{X}^nY^n}$ without explicitly referencing $W$.

\subsection{Rate-Distortion Function $R(m, d,\mathcal{A})$}
\label{apdx:subsec_bridge_rate_dist_others_nlc}

While an intuitive interpretation of the rate-distortion function $R(m, d,\mathcal{A})$ was provided in Section \ref{apdx:subsec:nlc_detail_rate_distortion_function}, in this section, we explain the relationship between $R(m, d,\mathcal{A})$ and other theories based on Eq. \eqref{eq:I-TH_THU} of Theorem \ref{thm:nlc_rate_dist_MI_transform}.
From Eq. \eqref{eq:I-TH_THU}, $R(m, d,\mathcal{A})$ can be expressed as follows:
\begin{align}
    R(m, d,\mathcal{A}) = \mathbb{I}(U;H^{\mathcal{S}^*}_{\mathcal{A}}) = \mathbb{I}(T;H^{\mathcal{S}^*}_{\mathcal{A}}) - \mathbb{I}(T;H^{\mathcal{S}^*}_{\mathcal{A}}|U).
\label{eq:rep_R_MI_nlc}
\end{align}

Similar to the case in \cite{sugiyama2026finite}, the first term on the right-hand side of Eq. \eqref{eq:rep_R_MI_nlc}, $\mathbb{I}(T;H^{\mathcal{S}^*}_{\mathcal{A}})$, relates to the IT-bound proposed by \cite{xu2017information}.
Specifically, \cite{xu2017information} derives an upper bound on the generalization error using $\mathbb{I}(T;H^{\mathrm{S}^{\mathrm{i.i.d.}}}_{\mathcal{A}})$ under i.i.d. sampling as follows:
\begin{theorem}[Theorem 1 in \cite{xu2017information}]
   Assume that the training dataset $T$ contains $n$ samples drawn i.i.d. from $P^*_{\bm{X}Y}$.
   If $l(h(\bm{X}),Y)$ is $\sigma$-subgaussian for any $h \in \mathcal{H}$ under $P^*_{\bm{X}Y}$, then the following holds:
    \begin{align}
        |\mathrm{average~ test~ error} - \mathrm{average~ train~ error}| \le \sqrt{\frac{2\sigma^2}{n}\mathbb{I}(T;H^{\mathrm{S}^{\mathrm{i.i.d.}}}_{\mathcal{A}})}.
    \end{align}
    Here, $\mathbb{I}(T;H^{\mathrm{S}^{\mathrm{i.i.d.}}}_{\mathcal{A}})$ in the above inequality is calculated with respect to $P^{\mathcal{A}}_{H|T}P^*_{\bm{X}^n Y^n}$.
\end{theorem}

The only difference between $\mathbb{I}(T;H^{\mathrm{S}^{\mathrm{i.i.d.}}}_{\mathcal{A}})$ in \cite{xu2017information} and $\mathbb{I}(T;H^{\mathcal{S}^*}_{\mathcal{A}})$ in Eq. \eqref{eq:rep_R_MI_nlc} of the proposed framework lies in the method of acquiring the training dataset $T$.
While \cite{xu2017information} calculates $\mathbb{I}(T;H^{\mathrm{S}^{\mathrm{i.i.d.}}}_{\mathcal{A}})$ assuming $T$ is obtained using i.i.d. sampling $\mathrm{S}^{\mathrm{i.i.d.}}$, the proposed framework calculates the corresponding mutual information where $T$ is obtained by observing the pool $U$ and subsequently selecting data from $U$ according to the data selection strategy $\mathcal{S}^*$.
Therefore, $\mathbb{I}(T;H^{\mathcal{S}^*}_{\mathcal{A}})$ can be regarded as the counterpart of $\mathbb{I}(T;H^{\mathrm{S}^{\mathrm{i.i.d.}}}_{\mathcal{A}})$ from \cite{xu2017information} in the context of pool-based AL.

As noted in \cite{sugiyama2026finite}, the error bound using $\mathbb{I}(T;H^{\mathcal{S}^*}_{\mathcal{A}})$ provided by \cite{xu2017information} has been applied to the analysis of various methods.
Specifically, generalization error analysis is performed by evaluating $\mathbb{I}(T;H^{\mathrm{S}^{\mathrm{i.i.d.}}}_{\mathcal{A}})$ for the specific method under analysis and utilizing the aforementioned error bound.
Practical examples include the analysis of stochastic gradient Langevin dynamics (SGLD) \cite{welling2011bayesian} presented in \cite{pensia2018generalization,futami2023time}.

\subsection{The First Term of the Rate-Dispersion Function: $V_{\mathrm{in}}(m,d,\mathcal{A})$}
\label{apdx:subsec_bridge_Vin_others_nlc}

Section \ref{subsec:nlc_label_comp_bound} outlined that $V_{\mathrm{in}}$ characterizes the degree of overfitting of $\mathcal{A}$ and is related to IT-bounds and stability theory.
In this section, we provide a more detailed explanation of the relationship between $V_{\mathrm{in}}$ and these theories.
This relationship also inherits the advantages of the framework in \cite{sugiyama2026finite}.
Therefore, while the following explanation is based on \cite{sugiyama2026finite}, we restate it here to account for differences such as the existence of the pool, reflecting our focus on pool-based AL.

The decomposition of $V_{\mathrm{in}}$ presented in Section \ref{subsec:nlc_label_comp_bound} is reproduced below:
\begin{align}
    &V_{\mathrm{in}}(m,d,\mathcal{A}) 
    =\mathbb{E}_{P_{W}}[ 
    V_{\mathrm{in}, U}^{\iota} + V_{\mathrm{in}, \mathcal{S}}^{\iota} + V_{\mathrm{in}, \mathcal{A}}^{\iota} \notag \\
    &+(\lambda^*_{\mathcal{A}}(d))^2( V_{\mathrm{in}, U}^{\mathsf{d}} + V_{\mathrm{in}, \mathcal{S}}^{\mathsf{d}} + V_{\mathrm{in}, \mathcal{A}}^{\mathsf{d}} 
     )+ 2\lambda^{*}_{\mathcal{A}}(d) V_{\mathrm{in}}^{ \mathrm{cov}}], \notag\\
    &\text{where~}
     V_{\mathrm{in}, U}^{\iota} :=\mathrm{Var}_{P_{U|W}}( \mathbb{E}_{P^{\mathcal{A}}_{H|T} P^{\mathcal{S}^*}_{T|U}}[\iota_{U;H^{\mathcal{S}^*}_{\mathcal{A}}}(U;H)] ) , \notag\\
     &V_{\mathrm{in}, \mathcal{S}}^{\iota} :=\mathbb{E}_{P_{U|W}} [ \mathrm{Var}_{P^{\mathcal{S}^*}_{T|U}} ( \mathbb{E}_{P^{\mathcal{A}}_{H|T}} [\iota_{U;H^{\mathcal{S}^*}_{\mathcal{A}}}(U;H) ]) ] , \notag\\
     &V_{\mathrm{in}, \mathcal{A}}^{\iota} :=\mathbb{E}_{P_{U|W}P^{\mathcal{S}^*}_{T|U}}[ \mathrm{Var}_{P^{\mathcal{A}}_{H|T}}(\iota_{U;H^{\mathcal{S}^*}_{\mathcal{A}}}(U;H) )], \notag\\
     &V_{\mathrm{in}, U}^{\mathsf{d}} := \mathrm{Var}_{P_{U|W}}( \mathbb{E}_{P^{\mathcal{A}}_{H|T} P^{\mathcal{S}^*}_{T|U}} [\mathsf{d}(W;H)] ), \notag\\
     &V_{\mathrm{in}, \mathcal{S}}^{\mathsf{d}} := \mathbb{E}_{P_{U|W}}[ \mathrm{Var}_{P^{\mathcal{S}^*}_{T|U}} ( \mathbb{E}_{P^{\mathcal{A}}_{H|T}}[\mathsf{d}(W;H) ] ) ], \notag\\
     &V_{\mathrm{in}, \mathcal{A}}^{\mathsf{d}} := \mathbb{E}_{P_{U|W}P^{\mathcal{S}^*}_{T|U}}[ \mathrm{Var}_{P^{\mathcal{A}}_{H|T}}(\mathsf{d}(W;H) )], \notag\\
     &V_{\mathrm{in}}^{ \mathrm{cov}} :=  \mathrm{Cov}_{P^{\mathcal{A}}_{H|T}P^{\mathcal{S}^*}_{T|U}P_{U|W}} (\iota_{U;H^{\mathcal{S}^*}_{\mathcal{A}}}(U;H), \mathsf{d}(W;H) ). \notag
\end{align}

\subsubsection{Relationship of $V_{\mathrm{in}, U}^{\iota}$, $V_{\mathrm{in}, \mathcal{S}}^{\iota}$, and $V_{\mathrm{in}, \mathcal{A}}^{\iota}$ with Other Theories}
\label{apdx:subsubsec:V_in_iota_other_theories}

In the following explanation, to describe the components of $V_{\mathrm{in}}$, we focus on the term inside $\mathbb{E}_{P_W}[\cdot]$ and consider a fixed $w$.
The sum of the three terms $V_{\mathrm{in}, U}^{\iota}$, $V_{\mathrm{in}, \mathcal{S}}^{\iota}$, and $V_{\mathrm{in}, \mathcal{A}}^{\iota}$ defined below as $V_{\mathrm{in}}^{\iota}$, which measure the variance of $\iota_{U;H^{\mathcal{S}^*}_{\mathcal{A}}}(U;H)$, is related to the term $\mathrm{Var}(T;H^{\mathrm{S}^{\mathrm{i.i.d.}}}_{\mathcal{A}})$ that evaluates the degree of overfitting in the IT-bound proposed by \cite{hellstrom2020generalization}:
\begin{align}
\begin{split}
    V_{\mathrm{in}}^{\iota} 
    &:= V_{\mathrm{in}, U}^{\iota} +  V_{\mathrm{in}, \mathcal{S}}^{\iota} + V_{\mathrm{in}, \mathcal{A}}^{\iota} \\
    &:= \mathrm{Var}_{P^{\mathcal{A}}_{H|T} P^{\mathcal{S}^*}_{T|U}P_{U|W=w}}( \iota_{U;H^{\mathcal{S}^*}_{\mathcal{A}}}(U;H) ) 
\end{split}
\end{align}

\cite{hellstrom2020generalization} provides an upper bound on the generalization error as follows:
\begin{theorem}[A special case of Corollary 6 in \cite{hellstrom2020generalization}]
    Assume that the training dataset $T$ consists of $n$ samples drawn i.i.d. from $P^*_{\bm{X}Y}$.
    Assume that $l(h(\bm{X}),Y)$ is $\sigma$-subgaussian for any $h \in \mathcal{H}$ under $P^*_{\bm{X}Y}$.
    Furthermore, assume that $P^{\mathcal{A}}_{H|T}P^*_{\bm{X}^nY^n}$ is absolutely continuous with respect to $P^{\mathcal{A}}_{H}P^*_{\bm{X}^nY^n}$, where $P^{\mathcal{A}}_{H}$ is the marginal distribution obtained by marginalizing $P^{\mathcal{A}}_{H|T}$ over $T\sim P^*_{\bm{X}^nY^n}$.
    Then, for any $\delta \in (0,1)$, the following holds with probability at least $1-\delta$:
    \begin{align}
        |\mathrm{average~ test~ error} - \mathrm{average~ train~ error}| \le \sqrt{\frac{2\sigma^2}{n}\bigg( \mathbb{I}(T;H^{\mathrm{S}^{\mathrm{i.i.d.}}}_{\mathcal{A}}) + \frac{\mathrm{Var}(T;H^{\mathrm{S}^{\mathrm{i.i.d.}}}_{\mathcal{A}})}{(\delta/2)^{1/2}}  + \log \frac{2}{\delta}\bigg)}
    \label{eq:hellstrom2020thm6}
    \end{align}
    where $\mathrm{Var}(T;H^{\mathrm{S}^{\mathrm{i.i.d.}}}_{\mathcal{A}}):=\mathrm{Var}_{P^{\mathcal{A}}_{H|T}P^*_{\bm{X}^nY^n} }\Big(\iota_{T;H^{\mathrm{S}^{\mathrm{i.i.d.}}}_{\mathcal{A}}}(T;H) \Big)$.
\label{thm:hellstrom2020thm6}
\end{theorem}

We describe the relationship between $V_{\mathrm{in}}^{\iota}$ and $\mathrm{Var}(T;H^{\mathrm{S}^{\mathrm{i.i.d.}}}_{\mathcal{A}})$.
First, by definition, $\iota_{U;H^{\mathcal{S}^*}_{\mathcal{A}}}(u;h)$ can be decomposed as follows:
\begin{align}
    \textstyle \iota_{U;H^{\mathcal{S}^*}_{\mathcal{A}}}(u;h)
    = \log \frac{\mathrm{d} P^{\mathcal{A},\mathcal{S}^*}_{H|U=u}}{\mathrm{d} P^{\mathcal{A},\mathcal{S}^*}_{H}}(h) 
    = \underbracket[0.5pt]{\log \frac{\mathrm{d} P^{\mathcal{A}}_{H|T=t}}{\mathrm{d} P^{\mathcal{A},\mathcal{S}^*}_{H}}(h) }_{=\iota_{T;H^{\mathcal{S}^*}_{\mathcal{A}}}(t;h) }
     + \underbracket[0.5pt]{\log \frac{\mathrm{d} P^{\mathcal{A},\mathcal{S}^*}_{H|U=u}}{\mathrm{d} P^{\mathcal{A}}_{H|T=t}}(h) }_{=:\Delta(u;t;h)},
\end{align}
where $P^{\mathcal{A},\mathcal{S}^*}_{H|U=u}(h) = P^{\mathcal{A},\mathcal{S}^*}_{H|U}(h|u) = \sum_{t \in \mathcal{T}}P^{\mathcal{A}}_{H|T}(h|t)P^{\mathcal{S}^*}_{T|U}(t|u)$ and $P^{\mathcal{A},\mathcal{S}^*}_{H}(h)=\sum_{w\in\mathcal{W}}\sum_{u \in \mathcal{U}(w)} P^{\mathcal{A},\mathcal{S}^*}_{H|U}(h|u)P_{W|U}(u|w)P_W(w)$.
Here, let $\iota_{U;H^{\mathcal{S}^*}_{\mathcal{A}}}(u;t;h)=\iota_{T;H^{\mathcal{S}^*}_{\mathcal{A}}}(t;h)+\Delta(u;t;h)$. Using this to expand $V_{\mathrm{in}}^{\iota}$ yields the following expression:
\begin{align}
    V_{\mathrm{in}}^{\iota}
    &=\mathrm{Var}_{P^{\mathcal{A}}_{H|T} P^{\mathcal{S}^*}_{T|U}P_{U|W=w}}\Big( \iota_{U;H^{\mathcal{S}^*}_{\mathcal{A}}}(U;H) \Big)  \notag \\
    &= \mathrm{Var}_{P^{\mathcal{A}}_{H|T} P^{\mathcal{S}^*}_{T|U}P_{U|W=w}} \Big( \iota_{T;H^{\mathcal{S}^*}_{\mathcal{A}}}(T;H)+\Delta(U;T;H) \Big)  \notag \\
    &=\mathrm{Var}_{P^{\mathcal{A}}_{H|T} P^{\mathcal{S}^*}_{T|U}P_{U|W=w}} \Big(\iota_{T;H^{\mathcal{S}^*}_{\mathcal{A}}}(T;H) \Big)
        + \mathrm{Var}_{P^{\mathcal{A}}_{H|T} P^{\mathcal{S}^*}_{T|U}P_{U|W=w}}\Big(\Delta(U;T;H)\Big) \notag \\
    &~~~ + 2 \mathrm{Cov}_{P^{\mathcal{A}}_{H|T} P^{\mathcal{S}^*}_{T|U}P_{U|W=w}}\Big(\iota_{T;H^{\mathcal{S}^*}_{\mathcal{A}}}(T;H), \Delta(U;T;H) \Big). \notag
\end{align}
By definition, $\Delta(U;T;H)$ approaches zero when $T$ contains sufficient information about $U$.
Therefore, when the required distortion level $d$ is sufficiently small and $T\sim P^{\mathcal{S}^*}_{T|U}$ captures sufficient information about $U$, it can be inferred that $\Delta(U;T;H) \approx 0$.
Consequently, $\mathrm{Var}_{P^{\mathcal{A}}_{H|T}P^{\mathcal{S}^*}_{T|U}P_{U|W=w}}(\Delta(U;T;H))$ is considered to be negligibly small.
Furthermore, as $P^{\mathcal{A}}_{H|T=t}(h)$ increases, $\iota_{T;H^{\mathcal{S}^*}_{\mathcal{A}}}(T;H)$ increases while $\Delta(U;T;H)$ decreases; thus, the covariance term $\mathrm{Cov}_{P^{\mathcal{A}}_{H|T}P^{\mathcal{S}^*}_{T|U}P_{U|W=w}}(\iota_{T;H^{\mathcal{S}^*}_{\mathcal{A}}}(T;H), \Delta(U;T;H))$ is considered to be non-positive ($\le 0$).
Therefore, the dominant term of $V_{\mathrm{in}}^{\iota}$ is $\mathrm{Var}_{P^{\mathcal{A}}_{H|T} P^{\mathcal{S}^*}_{T|U}P_{U|W=w}} (\iota_{T;H^{\mathcal{S}^*}_{\mathcal{A}}}(T;H))$, and $V_{\mathrm{in}}^{\iota}$ is considered to be upper bounded by this term.

This dominant term of $V_{\mathrm{in}}^{\iota}$, namely $\mathrm{Var}_{P^{\mathcal{A}}_{H|T} P^{\mathcal{S}^*}_{T|U}P_{U|W=w}} (\iota_{T;H^{\mathcal{S}^*}_{\mathcal{A}}}(T;H))$, is highly similar to $\mathrm{Var}(T;H^{\mathrm{S}^{\mathrm{i.i.d.}}}_{\mathcal{A}})$ in Eq. \eqref{eq:hellstrom2020thm6}, differing only in the process by which the training dataset $T$ is obtained.
While \cite{hellstrom2020generalization} assumes i.i.d. sampling, the proposed framework assumes the observation of a pool $U$ and the selection of a training dataset from $U$.
In other words, $V_{\mathrm{in}}^{\iota}$ can be regarded as the counterpart of $\mathrm{Var}(T;H^{\mathrm{S}^{\mathrm{i.i.d.}}}_{\mathcal{A}})$ in \cite{hellstrom2020generalization} within the context of pool-based AL.

\subsubsection{Relationship of $V_{\mathrm{in}, \mathcal{S}}^{\mathsf{d}}$ and $V_{\mathrm{in}, \mathcal{A}}^{\mathsf{d}}$ with Other Theories}
\label{apdx:subsubsec:V_in_dist_other_theories}

The terms $V_{\mathrm{in}, \mathcal{S}}^{\mathsf{d}}$ and $V_{\mathrm{in}, \mathcal{A}}^{\mathsf{d}}$ measure the extent to which the distortion $\mathsf{d}(W;H)$ varies due to variations in the training dataset $T$ and the randomness inherent in $\mathcal{A}$, respectively.
This is related to the concept of uniform stability in stability theory \cite{bousquet2002stability,elisseeff2005stability}.

First, we briefly review the relevant stability theory.
Uniform stability with respect to variations in the training dataset is defined as follows:
\begin{definition}[Definition 13 in \cite{elisseeff2005stability}]
    A randomized learning algorithm $\mathcal{A}$ is said to have uniform stability $\beta$ with respect to a loss function $l:\mathcal{H}\times\mathcal{X}\times\mathcal{Y}\rightarrow \mathbb{R}_+$ if the following holds:
    \begin{align}
    \forall i \in \{1,\dots,n\}, \sup_{t,(\bm{x},y)}\Big|\mathbb{E}_{P^{\mathcal{A}}_{H|T=t}}[l(H,(\bm{x},y))] - \mathbb{E}_{P^{\mathcal{A}}_{H|T=t^{\setminus i}}}[l(H,(\bm{x},y))]  \Big| \le \beta_n,
    \label{eq:uniform_stability_data}
    \end{align}
where, $t^{\setminus i}=t \setminus \{(\bm{x}_i, y_i)\}$.
\label{def:uniform_stability_data}
\end{definition}
Note that this definition is not unique. In addition to the difference between $t$ and $t^{\setminus i}$, it is sometimes defined using the difference between $t$ and $t^{i}$, where the $i$-th sample $(\bm{x}_i,y_i)$ of $t$ is replaced by a new sample $(\bm{x}',y')\sim P^*_{\bm{X}Y}$ \cite{kuzbnorskij2018data}.
As discussed in \cite{sugiyama2026finite}, the choice of definition does not affect the relationship with our framework described below.

Furthermore, the quantity measuring the variation in the prediction error of the learned hypothesis due to the randomness of $\mathcal{A}$ is defined as follows \cite{elisseeff2005stability}:
\begin{align}
    \sup_{r_1,\dots,r_B, r'_b} \sup_{(\bm{x},y)} \Big| l(h_{T, (r_1,\dots,r_B)}, (\bm{x},y)) - l(h_{T, (r_1,\dots,r_{b-1}, r'_b, r_{b+1},\dots,r_B)}, (\bm{x},y))  \Big| \le \rho.
\label{eq:uniform_stability_random}
\end{align}
Here, $\bm{r} = (r_1,\dots,r_{B}) \in \mathbb{R}^{B}$ with $B \in \mathbb{N}_+$ is a random variable representing the random component of $\mathcal{A}$.
Each element of $\bm{r}$ is assumed to be drawn independently from a distribution $P_r$.
The notation $h_{T,(r_1,\dots,r_B)}$ denotes the hypothesis learned from the random component $\bm{r}$ and the dataset $T$.

In this stability theory, the generalization error is upper bounded using the aforementioned $\beta_n$ and $\rho$ as follows.
\begin{theorem}[Theorem 15 in \cite{elisseeff2005stability}]
    Assume that $\mathcal{A}$ has uniform stability $\beta_n$ with respect to a loss function $l$ satisfying $0\le l(h,(\bm{x},y))\le L$, and that there exists $\rho$ satisfying Eq. \eqref{eq:uniform_stability_random} for any $b$.
    Here, $L \in \mathbb{R}_+$ is the maximum value of $l$.
    Then, for any $n \ge 1$ and any $\delta \in (0,1)$, the following holds with probability at least $1-\delta$:
    \begin{align}
        \mathrm{average~ test~ error} - \mathrm{average~ train~ error} \le 2\beta_n +\bigg(\frac{L+4n \beta_n}{\sqrt{2n}} + \sqrt{2B}\rho \bigg)\sqrt{\log 2 /\delta}.
    \end{align}
\end{theorem}

The parameters $\beta_n$ and $\rho$ relate to $V_{\mathrm{in}, \mathcal{S}}^{\mathsf{d}}$ and $V_{\mathrm{in}, \mathcal{A}}^{\mathsf{d}}$, respectively.
Specifically, the quantities upper bounded by $\beta_n$ and $\rho$ in Eqs. \eqref{eq:uniform_stability_data} and \eqref{eq:uniform_stability_random}, respectively, are related to $V_{\mathrm{in}, \mathcal{S}}^{\mathsf{d}}$ and $V_{\mathrm{in}, \mathcal{A}}^{\mathsf{d}}$.

First, we provide a detailed explanation of the relationship between $\beta_n$ and $V_{\mathrm{in}, \mathcal{S}}^{\mathsf{d}}$ based on \cite{sugiyama2026finite}.
Stability theory assumes i.i.d. sampling. 
From the main result (Eq. (2.1)) of \cite{steele1986efron}, the following holds:
\begin{align}
    \mathrm{Var}_{P^{\mathcal{S}^{\mathrm{iid}}}_{T|W=w}} ( \mathbb{E}_{P^{\mathcal{A}}_{H|T}} [ \mathsf{d}(w;H) ]) 
    \le \frac{1}{2}\sum^n_{i=1}\mathbb{E}_{P^{\mathcal{S}^{\mathrm{iid}}}_{T|W=w}P^*_{\bm{X}Y}}\Big[\Big(\mathbb{E}_{P^{\mathcal{A}}_{H|T}} [ \mathsf{d}(w;H) ] - \mathbb{E}_{P^{\mathcal{A}}_{H|T^{i}}} [ \mathsf{d}(w;H)] \Big)^2 \Big].
\label{eq:var_3_iid_dist_steele}
\end{align}
Here, $T^{i}$ is defined by replacing the $i$-th sample $(\bm{X}_i,Y_i)$ of $T$ with a new sample $(\bm{X}',Y') \sim P^*_{\bm{X}Y}$.
Now, we replace $l(H,(\bm{x},y))$ in the left-hand side of Eq. \eqref{eq:uniform_stability_data} with $\mathsf{d}(w;H)$ as follows:
\begin{align}
    \forall i \in \{1,\dots,n\}, \sup_{t,w}\Big|\mathbb{E}_{P^{\mathcal{A}}_{H|T=t}}[\mathsf{d}(w;H)] - \mathbb{E}_{P^{\mathcal{A}}_{H|T=t^{\setminus i}}}[\mathsf{d}(w;H)]  \Big| \le \beta_n.
\end{align}
From this expression, we obtain:
\begin{align}
    &\Big|\mathbb{E}_{P^{\mathcal{A}}_{H|T}} [ \mathsf{d}(w;H) ] - \mathbb{E}_{P^{\mathcal{A}}_{H|T^{i}}} [ \mathsf{d}(w;H)]\Big| \notag \\
    &\le \Big|\mathbb{E}_{P^{\mathcal{A}}_{H|T}} [ \mathsf{d}(w;H) ] - \mathbb{E}_{P^{\mathcal{A}}_{H|T^{\setminus i}}} [ \mathsf{d}(w;H)]\Big| 
    + \Big|\mathbb{E}_{P^{\mathcal{A}}_{H|T^{i}}} [ \mathsf{d}(w;H) ] - \mathbb{E}_{P^{\mathcal{A}}_{H|T^{\setminus i}}} [ \mathsf{d}(w;H)] \Big| \notag \\
    &\le 2\beta_n.
\label{eq:stability_triangle_transformation}
\end{align}
Applying this result to Eq. \eqref{eq:var_3_iid_dist_steele} yields:
\begin{align}
    \mathrm{Var}_{P^{\mathcal{S}^{\mathrm{iid}}}_{T|W=w}} ( \mathbb{E}_{P^{\mathcal{A}}_{H|T}} [ \mathsf{d}(w;H) ])
    \le 2n(\beta_n)^2.
\end{align}
Note that if $\beta_n$ bounds the difference between $T$ and $T^{i}$, the left-hand side of Eq. \eqref{eq:stability_triangle_transformation} is bounded by $\beta_n$, resulting in $\mathrm{Var}_{P^{\mathcal{S}^{\mathrm{iid}}}_{T|W=w}} ( \mathbb{E}_{P^{\mathcal{A}}_{H|T}} [ \mathsf{d}(w;H) ]) \le \frac{n}{2}\beta_n^2$.

Based on the above results, the variance $\mathrm{Var}_{P^{\mathcal{S}^{\mathrm{iid}}}_{T|W=w}} ( \mathbb{E}_{P^{\mathcal{A}}_{H|T}} [ \mathsf{d}(w;H) ])$, which is evaluated by the uniform stability $\beta_n$, corresponds to $V_{\mathrm{in}, \mathcal{S}}^{\mathsf{d}}$ in the proposed framework, differing only in the process by which the training dataset $T$ is obtained.
The former assumes i.i.d. sampling, whereas the latter is specific to pool-based AL.
Therefore, $V_{\mathrm{in}, \mathcal{S}}^{\mathsf{d}}$ in the proposed framework can be regarded as the counterpart of the variance evaluated by the uniform stability $\beta_n$ within the context of pool-based AL.

As noted in \cite{sugiyama2026finite}, the uniform stability $\beta_n$, which is related to the proposed framework, has been utilized in the theoretical analysis of SGD \cite{hardt2016train,kuzbnorskij2018data,bassily2020stability,lei2021stability,lei2023stability} and SGLD \cite{mou2018generalization}.
These analyses evaluate the generalization performance of SGD by deriving upper bounds on $\beta_n$.
Therefore, the established relationship with uniform stability suggests the potential to bridge these existing analyses with pool-based AL.

Next, we explain the relationship between $\rho$ and $V_{\mathrm{in}, \mathcal{A}}^{\mathsf{d}}$ in detail based on \cite{sugiyama2026finite}.
Similar to the case of $\beta_n$, the lower bound associated with $\rho$ in Eq. \eqref{eq:uniform_stability_random} relates to the following variance.
The following inequality is demonstrated by the main result (Eq. (2.1)) of \cite{steele1986efron}:
\begin{align}
\begin{split}
    &\mathbb{E}_{P^{\mathcal{S}^{\mathrm{iid}}}_{T|W=w}} [ \mathrm{Var}_{P^{\mathcal{A}}_{H|T}} ( \mathsf{d}(w;H) )] \\
    &\le \frac{1}{2}\sum^n_{i=1}\mathbb{E}_{P^{\mathcal{S}^{\mathrm{iid}}}_{T|W=w}}\Big[
    \mathbb{E}_{(P_r)^{B+1}}
    \Big[\Big( l(h_{T, (r_1,\dots,r_B)}, (\bm{x},y)) - l(h_{T, (r_1,\dots,r_{b-1}, r'_b, r_{b+1},\dots,r_B)}, (\bm{x},y)) \Big)^2 \Big] \Big].
\end{split}
\label{eq:var_4_iid_dist_steele}
\end{align}
Combining the above inequality with Eq. \eqref{eq:uniform_stability_random} yields the following:
\begin{align}
    \mathbb{E}_{P^{\mathcal{S}^{\mathrm{iid}}}_{T|W=w}} [ \mathrm{Var}_{P^{\mathcal{A}}_{H|T}} ( \mathsf{d}(w;H) )]
    \le \frac{n}{2}\rho^2.
\end{align}
Therefore, $\rho$ can be considered as a quantity that measures $\mathbb{E}_{P^{\mathcal{S}^{\mathrm{iid}}}_{T|W=w}} [ \mathrm{Var}_{P^{\mathcal{A}}_{H|T}} ( \mathsf{d}(w;H) )]$.

The difference between the quantity measured by $\rho$ and $V_{\mathrm{in}, \mathcal{A}}^{\mathsf{d}}$ in the proposed framework lies solely in the process of obtaining the training dataset $T$.
The former assumes i.i.d. sampling, whereas the latter assumes pool-based AL.
Consequently, $V_{\mathrm{in}, \mathcal{A}}^{\mathsf{d}}$ in the proposed framework can be regarded as the counterpart of the variance evaluated by $\rho$ in pool-based AL.

\section{Proof of Section~\ref{sec:nlc_analysis}}
\label{apdx:sec:proof_nlc_analysis}

\subsection{Proof of Theorem~\ref{thm:ML_nlc_eps_conv_bound}}
\label{apdx:subsec:proof_thm:ML_nlc_eps_conv_bound}

\begin{proof}

This proof is based on the proof technique of Theorem 1 in \cite{kostina2016nonasymptotic}.
Let $P^{\mathcal{A}}_{H|T}$ denote a learning algorithm, and let $P^{\mathcal{S}}_{T|U}$ define a data selection strategy characterized as an $(n,m, d,\epsilon, \mathcal{A})$-selection.
Here, we assume that $\mathrm{supp}(P^{\mathcal{S}}_{T|U}) \in \mathfrak{P}_n(U)$.
That is, $T$ is assumed to be a subset of $U$ with cardinality $n$.
We then define the following set:
\begin{align}
  B_d(w) := \{ h \in \mathcal{H} : \mathsf{d}(w,h) \le d \}.
\end{align}

For any $\gamma \ge 0$, the following holds:

\begin{align}
  &\mathbb{P}[\tilde{\jmath}_{W;U}(W,U,H,d,\mathcal{A}) \ge bn + \gamma] \notag \\
  &= \mathbb{P}[\tilde{\jmath}_{W;U}(W,U,H,d,\mathcal{A}) \ge bn + \gamma, \mathsf{d}(W;H) > d]
     + \mathbb{P}[\tilde{\jmath}_{W;U}(W,U,H,d,\mathcal{A}) \ge bn + \gamma, \mathsf{d}(W;H) \le d] \notag \\
  &\le \epsilon + \sum_{w\in\mathcal{W}} P_W(w) \sum_{u \in \mathcal{U}(w)} P_{U|W}(u|w) \sum_{t \in \mathfrak{P}_n(u)} P^{\mathcal{S}}_{T|U}(t|u) \sum_{h \in B_d(w)} P^{\mathcal{A}}_{H|T}(h|t) \mathbbm{1}_{[e^{bn} \le \exp(\tilde{\jmath}_{W;U}(w,u,h,d,\mathcal{A}) - \gamma)]}  \\
  &\le \epsilon + e^{-\gamma} \sum_{w\in\mathcal{W}} P_W(w) \sum_{u \in \mathcal{U}(w)} P_{U|W}(u|w) \sum_{t \in \mathfrak{P}_n(U)} \frac{1}{2^{bn}} \sum_{h \in B_d(w)} P^{\mathcal{A}}_{H|T}(h|t) e^{\tilde{\jmath}_{W;U}(w,u,h,d,\mathcal{A})} \label{eq:MLnlc_ecb_tmp2} \\
  &\le \epsilon + e^{-\gamma} \sum_{w\in\mathcal{W}} P_W(w) \sum_{u \in \mathcal{U}(w)} P_{U|W}(u|w) \sum_{t \in \mathfrak{P}_n(U)} \frac{1}{2^{bn}} \sum_{h \in \mathcal{H}} e^{\lambda^*_{\mathcal{A}}(d)d - \lambda^*_{\mathcal{A}}(d)\mathsf{d}(w;h)} P^{\mathcal{A}}_{H|T}(h|t)  e^{\tilde{\jmath}_{W;U}(w,u,h,d,\mathcal{A})} \label{eq:MLnlc_ecb_tmp3} \\
  &= \epsilon + e^{-\gamma} \sum_{w\in\mathcal{W}} P_W(w) \sum_{u \in \mathcal{U}(w)} P_{U|W}(u|w) \sum_{t \in \mathfrak{P}_n(U)} \frac{1}{2^{bn}} \sum_{h \in \mathcal{H}} P^{\mathcal{A}}_{H|T}(h|t)  e^{\iota_{U;H^{\mathcal{S}^*}_{\mathcal{A}}}(u;h) } \label{eq:MLnlc_ecb_tmp4} \\
  &= \epsilon + e^{-\gamma} \sum_{u \in \mathcal{U}} P_{U}(u) \sum_{w\in\mathcal{W}} P_{W|U}(w|u) \sum_{t \in \mathfrak{P}_n(U)} \frac{1}{2^{bn}} \sum_{h \in \mathcal{H}} P^{\mathcal{A}}_{H|T}(h|t)  e^{\iota_{U;H^{\mathcal{S}^*}_{\mathcal{A}}}(u;h) } \label{eq:MLnlc_ecb_tmp5} \\
  &= \epsilon + e^{-\gamma} \sum_{u \in \mathcal{U}} P_{U|H^{\mathcal{S}^*}_{\mathcal{A}}}(u|h) \sum_{w\in\mathcal{W}} P_{W|U}(w|u) \sum_{t \in \mathfrak{P}_n(U)} \frac{1}{2^{bn}} \sum_{h \in \mathcal{H}} P^{\mathcal{A}}_{H|T}(h|t)   \label{eq:MLnlc_ecb_tmp6} \\
    &= \epsilon + e^{-\gamma} \sum_{t \in \mathfrak{P}_n(U)} \frac{1}{2^{bn}} \sum_{h \in \mathcal{H}} P^{\mathcal{A}}_{H|T}(h|t) \sum_{u \in \mathcal{U}} P_{U|H^{\mathcal{S}^*}_{\mathcal{A}}}(u|h) \sum_{w\in\mathcal{W}} P^{\mathcal{S}^*}_{W|U}(w|u)   \label{eq:MLnlc_ecb_tmp7} \\
  &= \epsilon + e^{-\gamma}, \label{eq:MLnlc_ecb_tmp8} \\
\end{align}

where $P_{U}(u) = \sum_{w \in \mathcal{W}} P_{U|W}P_{W}$, $P_{W|U}(w|u) = P_{U|W}(u|w)P_{W}(w)/P_{U}(u)$, and $P_{U|H^{\mathcal{S}^*}_{\mathcal{A}}}(u|h)= (\sum_{w \in \mathcal{W}}\sum_{t \in \mathcal{T}} P^{\mathcal{A}}_{H|T}(h|t)P^{\mathcal{S}^*}_{T|U}(t|u)P_{U|W}(u|w)P_{W}(w)) / (\sum_{w \in \mathcal{W}} \sum_{u \in \mathcal{U}(w)} P^{\mathcal{A},\mathcal{S}^*}_{H|U}(h|u)P_{U|W}(u|w)P_{W}(w))$.
The first inequality follows from the fact that any $(n,m, d,\epsilon, \mathcal{A})$-selection satisfies $\mathbb{P}[\mathsf{d}(W;\mathcal{A}(\mathcal{S}(U))) > d] \le \epsilon$.
Furthermore, Eq. \eqref{eq:MLnlc_ecb_tmp2} utilizes the following inequality:
\begin{align}
  P^{\mathcal{S}}_{T|U}(t|u) \mathbbm{1}_{[e^{bn} \le \exp(\tilde{\jmath}_{W:U}(w,u,h,d,\mathcal{A}) - \gamma)]} 
  \le \frac{e^{-\gamma}}{e^{bn}}e^{\tilde{\jmath}_{W;U}(w,u,h,d,\mathcal{A})} 
  \le \frac{e^{-\gamma}}{2^{bn}}e^{\tilde{\jmath}_{W;U}(w,u,h,d,\mathcal{A})}.
\end{align}
For Eq. \eqref{eq:MLnlc_ecb_tmp3}, we applied:
\begin{align}
  \sum_{h \in B_d(w)} 1 
  = \sum_{h \in \mathcal{H}} \mathbbm{1}_{[\mathsf{d}(w;h) \le d]} 
  \le \sum_{h \in \mathcal{H}}  e^{\lambda^*_{\mathcal{A}}(d)d - \lambda^*_{\mathcal{A}}(d)\mathsf{d}(w;h)}.
\end{align}
Equation \eqref{eq:MLnlc_ecb_tmp4} uses the definition of $\tilde{\jmath}_{W;U}$, and Eq. \eqref{eq:MLnlc_ecb_tmp6} follows from the definition of the information density $\iota_{U;H^{\mathcal{S}^*}_{\mathcal{A}}}(u;h) = \log \frac{\mathrm{d} P_{U|H^{\mathcal{S}^*}_{\mathcal{A}}}(u|h) }{\mathrm{d} P_U(u)}$.
Finally, Eq. \eqref{eq:MLnlc_ecb_tmp8} relies on the fact that $|\mathfrak{P}_n(U)| = 2^{bn}$, given that a single sample can be represented by $b$ bits.

\end{proof}

\subsection{Proof of Theorem~\ref{thm:ML_nlc_rate_conv_bound}}
\label{apdx:subsec:proof_thm:ML_nlc_rate_conv_bound}

In this proof, we invoke the Berry-Esseen central limit theorem \cite{erokhin1958CLT} as stated below.
Similar to the proof of Theorem 4.5 in \cite{sugiyama2026finite}, this proof relies on finite blocklength analysis.

\begin{theorem}[Berry-Essen CLT \cite{erokhin1958CLT,kostina2012fixed}]

Let $k$ be a positive integer, and let $Z_i$ for $i=1,\dots,k$ be a sequence of independent random variables.
Then, for any real value $\alpha$, the following holds:
\begin{align}
  \bigg| \mathbb{P}\bigg[ \sum^k_{i=1} Z_i >  k \bigg( \mu_k + \alpha \sqrt{\frac{V_k}{k}} \bigg)  \bigg] - Q(t)  \bigg| \le \frac{B_k}{\sqrt{k}},
\label{eq:berry_essen_clt}
\end{align}
where,
\begin{align}
& \mu_k = \frac{1}{k}\sum^k_{i=1}\mathbb{E}[Z_i], \\
& V_k = \frac{1}{k}\sum^k_{i=1}\mathrm{Var}(Z_i), \\
& A_k = \frac{1}{k}\sum^k_{i=1}\mathbb{E}[|Z_i - \mu_i |^3], \\
& B_k = 6 \frac{A_k}{V^{3/2}_k}.
\end{align}
\label{thm:berry_essen_clt}
\end{theorem}

\begin{proof}

This proof is based on the proof technique of Theorem 12 in \cite{kostina2012fixed}.
To prove Theorem~\ref{thm:ML_nlc_rate_conv_bound}, we first define the $\mathcal{A}$-specified $d$-tilted information $\tilde{\jmath}_{W;U}(w^k,u^k, h^k,d,\mathcal{A})$ for $k>1$ as follows:
\begin{align}
    \tilde{\jmath}_{W;U}(w^k,u^k, h^k,h,\mathcal{A}) := \iota_{U;H^{\mathcal{S}^*}_{\mathcal{A}}}(u^k;h^k) + \lambda^*_{\mathcal{A}}(d) k(\mathsf{d}(w^k;h^k) - d).
\label{eq:ML_nlc_dfn_d-tilted_k}
\end{align}
where $h^k=(h_1,\dots,h_k) \in \mathcal{H}^k$.
When $k=1$, this definition coincides with $\tilde{\jmath}_{W;U}(w,u,h,d,\mathcal{A})$ defined in Eq. \eqref{eq:ML_nlc_dfn_d-tilted}.
Due to the i.i.d. property of $W$, the fact that $P_{U^k|W^k}=P_{U|W}\times\cdots\times P_{U|W}$, and Assumption (\Three), the relation $P^{\mathcal{A},\mathcal{S}^*}_{H^k|U^k} = P^{\mathcal{A},\mathcal{S}^*}_{H|U} \times \cdots \times P^{\mathcal{A},\mathcal{S}^*}_{H|U}$ holds for the optimal $P^{\mathcal{A},\mathcal{S}^*}_{H|W}$.
From these facts, for the first term on the right-hand side of Eq. \eqref{eq:ML_nlc_dfn_d-tilted}, we have $\iota_{U;H^{\mathcal{S}^*}_{\mathcal{A}}}(u^k;h^k)=\sum^k_{i=1}\iota_{U;H^{\mathcal{S}^*}_{\mathcal{A}}}(u_i;h_i)$.
Furthermore, since $\mathsf{d}(w^k;h^k)=\frac{1}{k}\sum^k_{i=1}\mathsf{d}(w_i;h_i)$, the second term on the right-hand side becomes $\lambda^*_{\mathcal{A}}(d) (\sum^k_{i=1}\mathsf{d}(w_i;h_i) - kd)$.
In other words, under these assumptions, the following holds for $\tilde{\jmath}_{W;U}(w^k,u^k,h^k,d,\mathcal{A})$:
\begin{align}
    \tilde{\jmath}_{W;U}(w^k,u^k,h^k,d,\mathcal{A}) = \sum^k_{i=1} \tilde{\jmath}_{W;U}(w_i, u_i,h_i,d,\mathcal{A}).
\end{align}
Moreover, under the definition given in Eq. \eqref{eq:ML_nlc_dfn_d-tilted}, Theorem \ref{thm:ML_nlc_eps_conv_bound} holds similarly for $k>1$.

First, we consider the case where $V(m, d, \mathcal{A}) > 0$.
To apply Theorem \ref{thm:berry_essen_clt} with $Z_i=\tilde{\jmath}_{W;U}(W_i,U_i, H^{\mathcal{S}^*}_{\mathcal{A},i},d,\mathcal{A})$, we define $A_k = \frac{1}{k}\sum^k_{i=1}\mathbb{E}[|\tilde{\jmath}_{W;U}(W_i,U_i,H^{\mathcal{S}^*}_{\mathcal{A},i},d,\mathcal{A}) - \mathbb{R}_W(m, d,\mathcal{A}) |^3]$ and $B_k = 6\frac{A_k}{(V(d,\mathcal{A}))^{3/2}}$.
Here, $H^{\mathcal{S}^*}_{\mathcal{A},i}$ is the hypothesis obtained from $U_i$ via $\mathcal{S}^*$ and $\mathcal{A}$.

In addition, in Theorem \ref{thm:ML_nlc_eps_conv_bound}, we let $\gamma = \frac{1}{2}\log k$ and configure the settings as follows:
\begin{align}
  bn &:= kR(m, d,\mathcal{A}) + \sqrt{k V(m, d,\mathcal{A})}Q^{-1}(\epsilon_k) - \gamma, \label{eq:MLnls_rcb_Km}\\
  \epsilon_k &:=  \epsilon + e^{-\gamma} + \frac{B_k}{\sqrt{k}}. \label{eq:MLnls_rcb_eps_n}
\end{align}
Applying Theorem \ref{thm:ML_nlc_eps_conv_bound}, the following holds for any $(k, m,n, d, \epsilon', \mathcal{A})$-selection:
\begin{align}
  \epsilon' 
  \ge \mathbb{P}\bigg[ \sum^k_{i=1}\tilde{\jmath}_{W;U}(W_i, U_i, H^{\mathcal{S}^*}_{\mathcal{A},i}, d,\mathcal{A}) \ge kR(m, d,\mathcal{A}) + \sqrt{kV(m, d,\mathcal{A})}Q^{-1}(\epsilon_k)  \bigg] - e^{-\gamma}.
\label{eq:MLnls_rcb_tmp1}
\end{align}

When $Z_i=\tilde{\jmath}_{W;U}(W_i,U_i,H^{\mathcal{S}^*}_{\mathcal{A},i},d,\mathcal{A})$, the parameters $\mu_k$ and $V_k$ in Theorem \ref{thm:berry_essen_clt} correspond to $\mathbb{R}_W(m, d,\mathcal{A})$ (which equals $R(m, d, \mathcal{A})$) and $V(m, d,\mathcal{A})$, respectively.
Consequently, the first term on the right-hand side of Eq. \eqref{eq:MLnls_rcb_tmp1} matches the left-hand side of Eq. \eqref{eq:berry_essen_clt} in Theorem \ref{thm:berry_essen_clt} when we set $\alpha=Q^{-1}(\epsilon_k)$.
Therefore, Theorem \ref{thm:berry_essen_clt} implies:
\begin{align}
  \mathbb{P}\bigg[ \sum^k_{i=1}\tilde{\jmath}_{W;U}(W_i,U_i,H^{\mathcal{S}^*}_{\mathcal{A},i}, d,\mathcal{A}) \ge kR(m, d,\mathcal{A}) + \sqrt{kV(m, d,\mathcal{A})}Q^{-1}(\epsilon_k)  \bigg] - Q(Q^{-1}(\epsilon_k))
  &\ge - \frac{B_k}{\sqrt{k}} \notag \\
  \mathbb{P}\bigg[ \sum^k_{i=1}\tilde{\jmath}_{W;U}(W_i,U_i,H^{\mathcal{S}^*}_{\mathcal{A},i}, d,\mathcal{A}) \ge kR(m, d,\mathcal{A}) + \sqrt{kV(m, d,\mathcal{A})}Q^{-1}(\epsilon_k)  \bigg]
  &\ge \epsilon + e^{-\gamma}.
\label{eq:MLnls_rcb_tmp2}
\end{align}

Substituting Eq. \eqref{eq:MLnls_rcb_tmp2} into Eq. \eqref{eq:MLnls_rcb_tmp1} yields $\epsilon' \ge \epsilon$.
Consequently, if $bn$ is specified as in Eq. \eqref{eq:MLnls_rcb_Km}, an excess distortion probability smaller than $\epsilon$ cannot be achieved.
That is, any $(k, m, n, d, \epsilon, \mathcal{A})$-selection satisfies:
\begin{align}
  bn \ge kR(m, d,\mathcal{A}) + \sqrt{k V(m, d,\mathcal{A})}Q^{-1}(\epsilon_k) - \gamma.
\label{eq:MLnls_rcb_tmp3}
\end{align}
Dividing both sides by $k$, we obtain:
\begin{align}
  \frac{bn}{k} \ge R(m, d,\mathcal{A}) + \sqrt{\frac{V(m, d,\mathcal{A})}{k}}Q^{-1}(\epsilon_k) - \frac{\gamma}{k}.
\label{eq:MLnls_rcb_tmp4}
\end{align}
Here, let $\epsilon_k = \epsilon + \Delta$ with $\Delta=e^{-\gamma} + \frac{B_k}{\sqrt{k}}$. The term $Q^{-1}(\epsilon_k)$ can be Taylor expanded as follows:
\begin{align}
  Q^{-1}(\epsilon_k) = +  Q^{-1}(\epsilon) + (Q^{-1})'(\epsilon) \times \Delta + \cdots.
\end{align}
Thus, the second term on the right-hand side of Eq. \eqref{eq:MLnls_rcb_tmp4} can be expanded as:
\begin{align}
  \sqrt{\frac{V(m,d,\mathcal{A})}{k}}Q^{-1}(\epsilon_n) 
  & = \sqrt{\frac{V(m, d,\mathcal{A})}{k}}Q^{-1}(\epsilon) + (Q^{-1})'(\epsilon) \sqrt{\frac{V(m,d,\mathcal{A})}{k}}\Delta + \cdots \notag \\
  &= \sqrt{\frac{V(m,d,\mathcal{A})}{k}}Q^{-1}(\epsilon)  + O\bigg(\frac{1}{k}\bigg). \notag
\end{align}
Here we used the property that $\Delta$ is of order $O(\frac{1}{\sqrt{k}})$ given $\gamma=\frac{1}{2}\log k$.
Applying this expansion to Eq. \eqref{eq:MLnls_rcb_tmp4}, the following holds for any $(k, m, n, d, \epsilon, \mathcal{A})$-selection:
\begin{align}
  \frac{bn}{k} \ge R(m, d,\mathcal{A}) + \sqrt{\frac{V(m, d,\mathcal{A})}{k}}Q^{-1}(\epsilon) - \frac{1}{2}\frac{\log k}{k} + O\bigg(\frac{1}{k}\bigg).
  \label{eq:MLnls_rcb_tmp5}
\end{align}
This concludes the proof of Eq. \eqref{eq:ML_nlc_rate_conv_bound} for the case where $V(m, d, \mathcal{A}) > 0$.

Next, we consider the case where $V(m, d, \mathcal{A}) = 0$.
In this case, $\tilde{\jmath}_{W;U}(W,U,H,d,\mathcal{A}) = R(m,d,\mathcal{A})$ holds almost surely.
Thus, by setting $\gamma = \log \frac{1}{1-\epsilon}$ and $bn = k R(m, d,\mathcal{A}) - \gamma$, and applying Theorem \ref{thm:ML_nlc_eps_conv_bound}, it follows that $\epsilon' \ge \epsilon$ for any $(k, m, n, d, \epsilon', \mathcal{A})$-selection.
Consequently, we reach the same conclusion as in the case where $V(m, d, \mathcal{A}) > 0$.

\end{proof}

\subsection{Proof of Theorem~\ref{thm:ML_nlc_distortion_conv_bound}}
\label{apdx:subsec:proof_thm:ML_nlc_distortion_conv_bound}

\begin{proof}

The proof of this theorem is substantially similar to that of Theorem 4.5 in \cite{sugiyama2026finite}.
In this proof, we derive Eq. \eqref{eq:ML_nlc_distortion_conv_bound} from Eq. \eqref{eq:ML_nlc_rate_conv_bound}.
First, we fix a point $(d_{\infty}, R_{\infty})$ on the rate-distortion curve $R(m, d,\mathcal{A})$, where $d_{\infty} \in (\underline{d}, \bar{d})$ and $R_{\infty}=R(m, d_{\infty},\mathcal{A})$.
We then define $d_k = D(k,m, R_{\infty}, \epsilon, \mathcal{A})$.
According to the results in \cite{kostina2012fixed} (Appendix E), it follows that $|d_k - d_{\infty}| \le O(\frac{1}{\sqrt{k}})$.

Since $R(k, m, \cdot, \epsilon, \mathcal{A})$ and $D(k, m,\cdot, \epsilon, \mathcal{A})$ are inverse functions, by the definition of $d_k$, we have $R_{\infty} = R(k, m, d_k, \epsilon, \mathcal{A})$.
Therefore, from Eq. \eqref{eq:ML_nlc_rate_conv_bound} in Theorem \ref{thm:ML_nlc_rate_conv_bound}, the following holds:
\begin{align}
  R_{\infty} \ge R(m,d_k, \mathcal{A}) + \sqrt{\frac{V(m,d_k,\mathcal{A})}{k}}Q^{-1}(\epsilon) + O\bigg(\frac{\log k}{k}\bigg).
\label{eq:Mls_nlc_dcb_tmp1}
\end{align}
Here, Taylor expanding $R(m, d_k, \mathcal{A})$ and $V(m,d_k,\mathcal{A})$ around $d_{\infty}$ yields:
\begin{align}
  & R(m,d_k, \mathcal{A}) = R(m,d_{\infty},\mathcal{A}) + R'(m,d_{\infty},\mathcal{A})(d_k - d_{\infty}) + O \bigg( \frac{1}{k} \bigg), \\
  & V(m,d_k, \mathcal{A}) = V(m,d_{\infty},\mathcal{A}) + O\bigg( \frac{1}{\sqrt{k}} \bigg).
\end{align}
Here, we used the property that $|d_k - d_{\infty}| \le O(\frac{1}{\sqrt{k}})$.
Substituting these expressions into Eq. \eqref{eq:Mls_nlc_dcb_tmp1} gives:
\begin{align}
  R_{\infty} \ge R(m,d_{\infty},\mathcal{A}) + R'(m,d_{\infty},\mathcal{A})(d_k - d_{\infty}) + \sqrt{\frac{V(m,d_{\infty},\mathcal{A})}{k}}Q^{-1}(\epsilon) + O\bigg(\frac{\log k}{k}\bigg).
\label{eq:Mls_nlc_dcb_tmp2}
\end{align}
Using the definition $R_{\infty}=R(m,d_{\infty},\mathcal{A})$ and the fact that $R'(m,d_{\infty},\mathcal{A}) < 0$, we can rearrange Eq. \eqref{eq:Mls_nlc_dcb_tmp2} to obtain:
\begin{align}
  d_k - d_{\infty} \ge - \frac{1}{R'(m,d_{\infty},\mathcal{A})}\sqrt{\frac{V(m,d_{\infty},\mathcal{A})}{k}}Q^{-1}(\epsilon)+ O\bigg(\frac{\log k}{k}\bigg).
  \label{eq:Mls_nlc_dcb_tmp3}
\end{align}
From the derivative of the inverse function, $D'(m,R_{\infty}, \mathcal{A}) = \frac{1}{R'(m,d_{\infty}, \mathcal{A})}$ holds.
Since $R(m,d,\mathcal{A}) =\mathbb{R}_W(m,d,\mathcal{A})$, it follows that $R'(m,d,\mathcal{A}) <0$ for any $d$.
From this fact and the definition of the $\mathcal{V}$, we have:
\begin{align}
  \sqrt{\mathcal{V}(m,R, \mathcal{A})} = |D'(m,R,\mathcal{A})|\sqrt{V(m,D(m,R), \mathcal{A})} = -D'(m,R,\mathcal{A}) \sqrt{V(m,D(m,R), \mathcal{A})}.
\end{align}
Applying these to Eq. \eqref{eq:Mls_nlc_dcb_tmp3}, we obtain:
\begin{align}
  &d_k - d_{\infty} \ge  \sqrt{\frac{\mathcal{V}(m,d_{\infty}, \mathcal{A})}{k}}Q^{-1}(\epsilon) + O\bigg(\frac{\log k}{k}\bigg) \notag \\
  &D(k, m,R_{\infty}, \epsilon, \mathcal{A}) \ge D(m,R_{\infty}, \mathcal{A}) + \sqrt{\frac{\mathcal{V}(m,d_{\infty}, \mathcal{A})}{k}}Q^{-1}(\epsilon) + + O\bigg(\frac{\log k}{k}\bigg).
\end{align}
Finally, we used the facts that $d_k = D(k, m,R_{\infty}, \epsilon, \mathcal{A})$ and $d_{\infty}=D(m,R_{\infty}, \mathcal{A})$.
Since the point $(d_{\infty}, R_{\infty})$ is an arbitrary point on the rate-distortion curve $R(m,d,\mathcal{A})$ satisfying $d_{\infty} \in (\underline{d}, \bar{d})$, Eq. \eqref{eq:ML_nlc_distortion_conv_bound} is proven.

\end{proof}

\subsection{Decomposition of the Rate-Dispersion Function $V(d,\mathcal{A})$}
\label{apdx:subsec:derive_V_decomposition}

In this section, we present the decompositions of $V(m, d,\mathcal{A})$, $V_{\mathrm{in}}$, and $V_{\mathrm{bet}}$.
First, the following is obtained by applying variance decomposition with respect to $P_W$:
\begin{align}
    V(m,d,\mathcal{A}) 
    & = \mathrm{Var}_{P^{\mathcal{A}}_{H|T}P^{\mathcal{S}^*}_{T|U}P_{U|W}P_W}(\tilde{\jmath}_{W;U}(W,U, H,m, d,\mathcal{A})) \notag \\
    & = V_{\mathrm{in}}(m,d,\mathcal{A}) + V_{\mathrm{bet}}(m,d,\mathcal{A}), \notag \\
    &\text{where~}V_{\mathrm{in}}(m,d,\mathcal{A}) :=\mathbb{E}_{P_W}\big[ \mathrm{Var}_{P^{\mathcal{A}}_{H|T}P^{\mathcal{S}^*}_{T|U}P_{U|W}} (\tilde{\jmath}_{W;U}(W,U, H, m, d,\mathcal{A})) \big], \notag \\
    &~~~~~~~~~~~V_{\mathrm{bet}}(m,d,\mathcal{A}) :=\mathrm{Var}_{P_W}\big( \mathbb{E}_{P^{\mathcal{A}}_{H|T}P^{\mathcal{S}^*}_{T|U}P_{U|W}}[\tilde{\jmath}_{W;U}(W,U, H, m, d,\mathcal{A})] \big). \notag
\end{align}

Next, the decomposition of $V_{\mathrm{in}}$ is derived by applying variance decomposition with respect to $P_{U|W}$, $P^{\mathcal{S}^*}_{T|U}$, and $P^{\mathcal{A}}_{H|T}$, together with the definition of $\tilde{\jmath}_{W;U}$:
\begin{align}
V_{\mathrm{in}}
&=\mathbb{E}_{P_W}\Big[ \mathrm{Var}_{P^{\mathcal{A}}_{H|T}P^{\mathcal{S}^*}_{T|U}P_{U|W}} \Big(\tilde{\jmath}_{W;U}(W,U, H,m, d,\mathcal{A}) \Big) \Big] \notag \\
&= \mathbb{E}_{P_W}\Big[ 
     \mathrm{Var}_{P^{\mathcal{A}}_{H|T}P^{\mathcal{S}^*}_{T|U}P_{U|W}} \Big(\iota_{U;H^{\mathcal{S}^*}_{\mathcal{A}}}(U;H) + \lambda^*_{\mathcal{A}}(d)(\mathsf{d}(W;H) -d)  \Big) 
     \Big]  \notag \\
&= \mathbb{E}_{P_W}\Big[ 
  \mathrm{Var}_{P^{\mathcal{A}}_{H|T}P^{\mathcal{S}^*}_{T|U}P_{U|W}} \Big(\iota_{U;H^{\mathcal{S}^*}_{\mathcal{A}}}(U;H) + \lambda^*_{\mathcal{A}}(d)\mathsf{d}(W;H)  \Big) 
      \Big] \notag \\
\begin{split}
  &= \mathbb{E}_{P_W}\Big[ 
    \mathrm{Var}_{P^{\mathcal{A}}_{H|T}P^{\mathcal{S}^*}_{T|U}P_{U|W}} \Big(\iota_{U;H^{\mathcal{S}^*}_{\mathcal{A}}}(U;H) \Big) 
    + (\lambda^*_{\mathcal{A}}(d))^2 \mathrm{Var}_{P^{\mathcal{A}}_{H|T}P^{\mathcal{S}^*}_{T|U}P_{U|W}} \Big(\mathsf{d}(W;H) \Big) \\
  &~~~~~~+ 2\lambda^*_{\mathcal{A}}(d) \mathrm{Cov}_{P^{\mathcal{A}}_{H|T}P^{\mathcal{S}^*}_{T|U}P_{U|W}} \Big(\iota_{U;H^{\mathcal{S}^*}_{\mathcal{A}}}(U;H), \mathsf{d}(W;H) \Big)
      \Big]
\end{split}
 \notag \\%\label{eq:ML_lossy_sc_V_in_3terms} \\
\begin{split}
  &= \mathbb{E}_{P_W}\Big[ 
      \mathrm{Var}_{P_{U|W}} \Big( \mathbb{E}_{P^{\mathcal{A}}_{H|T} P^{\mathcal{S}^*}_{T|U}}[\iota_{U;H^{\mathcal{S}^*}_{\mathcal{A}}}(U;H)] \Big) 
      + \mathbb{E}_{P_{U|W}} \Big[\mathrm{Var}_{P^{\mathcal{A}}_{H|T} P^{\mathcal{S}^*}_{T|U}} \Big(\iota_{U;H^{\mathcal{S}^*}_{\mathcal{A}}}(U;H) \Big) \Big] 
      \\
  &~~~~~ + (\lambda^*_{\mathcal{A}}(d))^2 \Big\{ 
       \mathrm{Var}_{P_{U|W}} \Big( \mathbb{E}_{P^{\mathcal{A}}_{H|T} P^{\mathcal{S}^*}_{T|U}} \Big[\mathsf{d}(W;H) \Big] )
      +  \mathbb{E}_{P_{U|W}} \Big[ \mathrm{Var}_{P^{\mathcal{A}}_{H|T} P^{\mathcal{S}^*}_{T|U}} \Big(\mathsf{d}(W;H) \Big) \Big]\Big\} \\
    &~~~~~ +  2\lambda^*_{\mathcal{A}}(d) \mathrm{Cov}_{P^{\mathcal{A}}_{H|T}P^{\mathcal{S}^*}_{T|U}P_{U|W}} \Big(\iota_{U;H^{\mathcal{S}^*}_{\mathcal{A}}}(U;H), \mathsf{d}(W;H) \Big)
  \Big].
\end{split}
\notag \\
\begin{split}
    &= \mathbb{E}_{P_W}\Big[ 
      \mathrm{Var}_{P_{U|W}} \Big( \mathbb{E}_{P^{\mathcal{A}}_{H|T} P^{\mathcal{S}^*}_{T|U}} \Big[\iota_{U;H^{\mathcal{S}^*}_{\mathcal{A}}}(U;H) \Big] \Big) 
      + \mathbb{E}_{P_{U|W}} \Big[ \mathrm{Var}_{P^{\mathcal{S}^*}_{T|U}} \Big( \mathbb{E}_{P^{\mathcal{A}}_{H|T}} \Big[\iota_{U;H^{\mathcal{S}^*}_{\mathcal{A}}}(U;H) \Big] \Big) \Big] \\
  &~~~~~~~~~~~~~~~~~~~+ \mathbb{E}_{P^{\mathcal{S}^*}_{T|U}P_{U|W}} \Big[ \mathrm{Var}_{P^{\mathcal{A}}_{H|T}}\Big(\iota_{U;H^{\mathcal{S}^*}_{\mathcal{A}}}(U;H) \Big)\Big]
      \\
  &~~~~~~~~~~~~~~~ + (\lambda^*_{\mathcal{A}}(d))^2 \Big\{ 
       \mathrm{Var}_{P_{U|W}}\Big( \mathbb{E}_{P^{\mathcal{A}}_{H|T} P^{\mathcal{S}^*}_{T|U}} \Big[\mathsf{d}(W;H) \Big] \Big)
      +  \mathbb{E}_{P_{U|W}} \Big[ \mathrm{Var}_{P^{\mathcal{S}^*}_{T|U}} \Big( \mathbb{E}_{P^{\mathcal{A}}_{H|T}} \Big[\mathsf{d}(W;H) \Big] \Big) \Big]  \\
  &~~~~~~~~~~~~~~~~~~~~~~~~~~~~~~~~~~~~~+ \mathbb{E}_{P^{\mathcal{S}^*}_{T|U}P_{U|W}} \Big[ \mathrm{Var}_{P^{\mathcal{A}}_{H|T}} \Big(\mathsf{d}(W;H) \Big) \Big]
      \Big\} \\
  &~~~~~~~~~~~~~~~ +  2\lambda^*_{\mathcal{A}}(d) \mathrm{Cov}_{P^{\mathcal{A}}_{H|T}P^{\mathcal{S}^*}_{T|U}P_{U|W}} \Big(\iota_{U;H^{\mathcal{S}^*}_{\mathcal{A}}}(U;H), \mathsf{d}(W;H) \Big)
  \Big].
\end{split} \notag 
\end{align}

Finally, the decomposition of $V_{\mathrm{bet}}$ is given by the definition of $\tilde{\jmath}_{W;U}$ as follows:
\begin{align}
    V_{\mathrm{bet}}(m,d,\mathcal{A})
    &= \mathrm{Var}_{P_W}\Big( \mathbb{E}_{P^{\mathcal{A}}_{H|T}P^{\mathcal{S}^*}_{T|U}P_{U|W}} \Big[\tilde{\jmath}_{W;U}(W,U, H, m, d,\mathcal{A}) \Big] \Big) \notag \\
    &= \mathrm{Var}_{P_W}\Big(\mathbb{E}_{P^{\mathcal{A}}_{H|T}P^{\mathcal{S}^*}_{T|U}P_{U|W}} \Big[ \iota_{W;H^{\mathcal{S}^*}_{\mathcal{A}}}(W;H) + \lambda^{*}_{\mathcal{A}}(d) (\mathsf{d}(W;H) - d) \Big]  \Big) \notag\\
    &= \mathrm{Var}_{P_W} \Big(  \mathbb{E}_{P^{\mathcal{A}}_{H|T}P^{\mathcal{S}^*}_{T|U}P_{U|W}}\Big[\iota_{W;H^{\mathcal{S}^*}_{\mathcal{A}}}(W;H)  \Big] \Big) 
     + \mathrm{Var}_{P_W} \Big(  \mathbb{E}_{P^{\mathcal{A}}_{H|T}P^{\mathcal{S}^*}_{T|U}P_{U|W}} \Big[\mathsf{d}(W;H) \Big] \Big) \notag \\
    &~~~~~~+ \mathrm{Cov}_{P_W} \Big( 
            \mathbb{E}_{P^{\mathcal{A}}_{H|T}P^{\mathcal{S}^*}_{T|U}P_{U|W}} \Big[\iota_{W;H^{\mathcal{S}^*}_{\mathcal{A}}}(W;H) \Big], \mathbb{E}_{P^{\mathcal{A}}_{H|T}P^{\mathcal{S}^*}_{T|U}P_{U|W}} \Big[\mathsf{d}(W;H) \Big] \Big). \notag
\end{align}